\documentclass{article}

\usepackage[preprint]{neurips_2023}

\setcitestyle{numbers,square}
\usepackage[utf8]{inputenc} 
\usepackage[T1]{fontenc}    
\usepackage{hyperref}       
\usepackage{url}            
\usepackage{booktabs}       
\usepackage{amsthm, amsmath, amsfonts, amssymb}       
\usepackage{nicefrac}       
\usepackage{microtype}      
\usepackage{xcolor}         
\usepackage[ruled]{algorithm2e}
\usepackage{graphicx}
\usepackage{subcaption}
\usepackage{multirow}
\usepackage{diagbox}

\newtheorem{thm}{\textbf{Theorem}}
\newtheorem{coro}{\textbf{Corollary}}
\newtheorem{assum}{\textbf{Assumption}}
\newtheorem{defn}{\textbf{Definition}}

\newtheorem{remark}{\textbf{Remark}}

\title{Modify Training Directions in Function Space to Reduce Generalization Error}

\author{
  Yi Yu\\
  Fudan University\\
  \texttt{y\_yu21@m.fudan.edu.cn} \\
 \And
 Wenlian Lu \\
 Fudan University \\
 \texttt{wenlian@fudan.edu.cn}\\
  \And
 Boyu Chen \\
 Fudan University \\
 \texttt{17110180037@fudan.edu.cn}\\
}

\begin{document}

\maketitle

\begin{abstract}
We propose theoretical analyses of a modified natural gradient descent method in the neural network function space based on the eigendecompositions of neural tangent kernel and Fisher information matrix. We firstly present analytical expression for the function learned by this modified natural gradient under the assumptions of Gaussian distribution and infinite width limit. Thus, we explicitly derive the generalization error of the learned neural network function using theoretical methods from eigendecomposition and statistics theory. By decomposing of the total generalization error attributed to different eigenspace of the kernel in function space, we propose a criterion for balancing the errors stemming from training set and the distribution discrepancy between the training set and the true data. Through this approach, we establish that modifying the training direction of the neural network in function space leads to a reduction in the total generalization error. Furthermore, We demonstrate that this theoretical framework is capable to explain many existing results of generalization enhancing methods. These theoretical results are also illustrated by numerical examples on synthetic data.
\end{abstract}

\section{Introduction}
Neural networks have achieved impressive success in tackling various challenging tasks appeared in real world. However, understanding the generalization performance of neural networks remains a complex and intricate problem for researchers. 
\par
Many factors affect generalization error of a model, such as the structure of neural network, the datasets utilized, the optimization algorithm chosen for training. A modern neural network always possessed more than millions of parameters, resulting in highly complex parameter space that make it extremely challenging to analyze their generalization error. However, a clearer perspective emerges when considering in the function space, since neural network is devoted to approximate the true model in a function space rather than parameter space. Recently, the seminal work of \cite{jacot2018neural} proved that in infinite width limit, the parameter-update based training dynamics can be converted to a differential dynamical system in function space. But conventional gradient descent optimization algorithms such as SGD \cite{bottou2012stochastic}, RMSProp \cite{tieleman2017divide}, Adam \cite{kingma2014method} are only operate directly in parameter space. Natural Gradient \cite{amari2000methods}, which utilizes curvature information in function space, is a gradient based optimization method which exhibits a strong connection with function space. In function space, the training dynamics of neural network can be interpreted as training in each eigenspace \cite{tancik2020fourier, bordelon2020spectrum}. Since different eigenspace associated with different spectrum contributes differently to the training dynamics\cite{bordelon2020spectrum} and consequently to the generalization error, there might exist an operation to modify the training dynamics in eigenspaces to enhance the generalization performance of the function learned. Building upon the aforementioned insights, we firstly propose an explicit solution of an over-parameterized neural network trained by Modified natural gradient descent (Modified NGD). Based on the explicit solution, we decompose the generalization error of the learned function into two components: one arising from the training set and the other from the distribution discrepancy between training set and true data. For the generalization error decomposition in each eigenspace, we balance the two error components and modify the training direction to reduce the generalization error.
\par
Several methods have been proposed to improve generalization performance, such as gradient suppression in cross domain generalization \cite{huang2020self}, self-distillation \cite{zhang2019your, mobahi2020self} and small batch training \cite{keskar2016large}. These methods can be incorporated into our theoratic framework to explain their efficacy. For these methods implicitly modify the eigenvalues of the Fisher Information matrix in function space and consequently the training direction in the eigenspaces of NTK. 
\section{Related Work}
Since neural network is a complex system, whose generalization error is difficult to track, it is reasonable to simplify the case to more trackable and representative models such as kernel method. With kernel method, there are many impressive results on generalization error bounds \cite{canatar2021spectral, bartlett2002rademacher, jacot2020kernel, loureiro2021learning, shawe2005eigenspectrum, bordelon2020spectrum}. The classical results of Bartlett \cite{bartlett2002rademacher} proved that the generalization error bound of kernel method is positively correlated with the trace of the kernel. Jacot et al. derived a risk estimator for kernel ridge regression \cite{jacot2020kernel}. \cite{loureiro2021learning} derived a closed-form genralization error of kernel regression for teacher-student distillation framework. \cite{shawe2005eigenspectrum} reveals that the convergence rate of kernel method. And \cite{bordelon2020spectrum} decomposes the average generalization error into eigencomponents under the Mercer's condition. 
\par
Since infinite width neural network forward process can be discribed by a so named Neural Network Gaussian process(NNGP) \cite{liu2020linearity}, thus wide neural network can be approximated by a linear model \cite{lee2019wide}. And a well known theoretical result of Jacot \cite{jacot2018neural} proved that in infinite width limit, the neural network is dominated by a constant kernel named neural tangent kernel(NTK), therefore the results on kernel methods can be applied to wide neural networks. Recently, the theoretical work of NTK is rich \cite{jacot2018neural, arora2019exact, liu2020linearity, geifman2020similarity, ortiz2021can, lee2019wide}, it is comfortable to derive theoretical analyses of generalization in NTK regime. Many work analyzed the effects of overparameterization on generalization error, such as overparameterization tends to converge to flat minima \cite{zhang2019your, keskar2016large}, easily escape from the local minima \cite{safran2021effects}, and some work view the generalization as compression \cite{arora2018stronger} and neuron unit-wise capacity \cite{liu2020toward} for analyzing. Based on these theoretical results, a lot of work take SGD into considering in NTK regime. \cite{velikanov2021explicit} analyze the training loss trajectory of SGD based on the spectrum of a loss operator, and many theoretical results on generalization error bounds of SGD is derived in NTK regime \cite{suzuki2018fast, cao2019generalization, allen2019learning, liu2022loss, liu2020toward}.
\par
Due to the high dimension complexity of parameter space of neural networks, the effect of SGD in parameter space is not explicit. Natural Gradient Descent (NGD), firstly proposed by Amari et al. \cite{amari2000methods}, considers the curvature information in function space. \cite{martens2020new} derived the connection between Fisher information matrix and Kullback-Leibler divergence in function space, proved that NGD is reparameterization invariant. In NTK regime, \cite{bernacchia2018exact} derived a explicit expression of the convergence rate of NGD in deep linear neural network, \cite{rudner2019natural} gives a analytic solution of NGD with linearization in infinite width limit, and \cite{karakida2020understanding} proves that under specific condition, existing approximate Fisher methods for NGD, such as K-FAC \cite{martens2015optimizing, grosse2016kronecker}, have the same convergence properties as exact NGD.
\par
In this paper, we leveraging the theoretical properties of NGD and NTK, give an analytical solution of Modifiedd NGD and derive an explicit decomposition of generalization error. Based on this decomposition, we modify the training directions of NGD in function space by modifying the Fisher information matrix to reduce the generalization error. We also discuss that our theoretical results can shed light on some existing generalization enhancing method, such as \cite{huang2020self, zhang2019your, mobahi2020self, keskar2016large}.

\section{Prelimilaries}
\subsection{Problem Setup}
Suppose the distribution of data points and labels is $p_{data}(x,y)$,where $x\in \mathbb{R}^{n_{in}}, y\in \mathbb{R}$ the training set $\{(x_i,y_i)\}_{i=1}^N \sim p_{data}(x,y)$, and the training data and the training label after vectorization is respectively $\mathcal{X} = (x_i)_{i=1}^N$ and $\mathcal{Y} = (y_i)_{i=1}^N$, then $\mathcal{X} \in \mathbb{R}^{N n_{in}}, \mathcal{Y} \in \mathbb{R}^{N}$. A fully connected neural network with $L$ layers whose width are respectively $\mathbb{R}^{n_{l}}, l = 1, 2, \dots, L$, can be expressed as:
\begin{equation}\label{NN}
f_\theta(x) = W^{(L)}\sigma\left(W^{(L-1)}\sigma\left(\dots \sigma\left(W^{(1)}x+b^{(1)}\right)\dots\right)+b^{(L-1)}\right)+b^{(L)}
\end{equation}
where $\sigma$ is the element-wise activation function, $\theta = \left(W^{(1)},W^{(2)}, \dots, W^{(L)}, b^{(1)}, b^{(2)}, \dots, b^{(L)}\right)$ is the weights of the network, $W^{(l)} \in \mathbb{R}^{n_{l-1} \times n_{l}}, b\in \mathbb{R}^{n_{l}}$ for $l=1,2,\dots ,L$, with $n_0 = n_{in}$ and $n_L = 1$.
\par
The purpose of training a neural network with training set $(\mathcal{X}, \mathcal{Y})$ is to minimize the loss function with respect to the parameter $\theta$ in a parameter space $\Theta \subset \mathbb{R}^P$, where $P={\sum_{l=1}^L(n_{l-1}+1)n_l}$:
\begin{equation}
\min\limits_{\theta \in \Theta} \mathcal{L}\left(f_\theta(\mathcal{X}), \mathcal{Y}\right).
\end{equation}
In the following sections, we take $\mathcal{L}$ to be $l_2$ loss, i.e. 
\begin{equation}
\mathcal{L}\left(f_\theta(\mathcal{X}), \mathcal{Y}\right) = \frac{1}{2}\|f_\theta(\mathcal{X})- \mathcal{Y}\|_2^2.
\end{equation}
\par
The generalization error, also known as expected risk, is defined as:
\begin{defn}\label{expectedRisk}
Suppose the data distribution is $p_{data}(x,y)$, and corresponding marginal distributions are $p_{data}(x)$ and $p_{data}(y)$, then for a predictor $f$, which maps the input $x$ to the output $y$, the expected risk of $f$ w.r.t. $l_2$ loss is 
\begin{equation}
\mathcal{R}(f) = \mathbb{E}_{(x,y)\sim p_{data}(x,y)}\left[\|f(x) - y\|_2^2\right].
\end{equation}
\end{defn}

\subsection{Natural Gradient Descent}
Let $\mathcal{H}$ be some function space, and $\mathcal{L}$ be a divergence, $f_\theta \in \mathcal{H}$ is a parameterized function, then the natural gradient under KL-divergence `metric' of $\mathcal{L}(f_\theta, y)$ at point $\theta$ is defined as
\begin{equation}\label{NG}
\tilde{\nabla}_{\theta} \mathcal{L} = \mathbf{F}^{-1} \nabla_\theta \mathcal{L}
\end{equation}
where $\mathbf{F} = \mathbb{E}_{x,y}\left[\nabla_\theta f_\theta(x,y)\nabla_\theta f_\theta(x,y)^\top\right]$ is the \textit{Fisher Information Matrix} of $f_\theta$.
\par
Natural Gradient Descent (NGD), defined based on natural gradient is an algorithm with parameter update rule that
\begin{equation}
\Delta \theta_t = - \eta\tilde{\nabla}_{\theta} \mathcal{L},
\end{equation}
where $\eta$ is the learning rate.

\subsection{Modified Natural Gradient Descent}
We propose a new natural gradient descent algorithm framework called \textit{Modified Natural Gradient Descent} (Modified NGD), for the ordinary natural gradient defined in \ref{NG}, denoting the non-zero singular values of $\mathbf{F}^{-1}$ as $\lambda_1\ge \lambda_2 \ge \dots \ge \lambda_n > 0$, then the Singular Value Decomposition(SVD) of $\mathbf{F}^{-1}$ can be expressed as:
\begin{equation}
\mathbf{F}^{-1} = \mathbf{V} \left(\mathbf{\Sigma} \oplus \mathbf{0}\right) \mathbf{V}^\top, \qquad \text{where} \; \mathbf{\Sigma} = diag(\lambda_1, \lambda_2 , \dots , \lambda_n), \; \mathbf{V}^\top \mathbf{V} =  \mathbf{V} \mathbf{V}^\top = \mathbf{I}
\end{equation}
The $\oplus$ operator in the above expression refers to direct sum of linear space. Let $c(\lambda)$ be some criterion with respect to the eigenvalue, then apply the modification operation $\varphi$ to the eigenvalues:
\begin{equation}\label{CutOperator}
\varphi: \lambda \mapsto  \begin{cases}
  0, &\text{ if }\; c(\lambda) \;\text{holds}   \\
  \lambda, &\text{ if }\; \text{otherwise}
\end{cases}
\end{equation}
Therefore, the eigenvalues matrix after modification operation is 
\begin{equation}
\mathbf{\Sigma}_\varphi = \varphi(\mathbf{\Sigma}) = diag(\varphi(\lambda_1), \varphi(\lambda_2) , \dots , \varphi(\lambda_n)).
\end{equation}
We reassemble  the modified matrix $\mathbf{\Sigma}_\varphi$ to be the factor of the inverse Fisher matrix resulting in \textit{Modified Inverse Fisher Information Matrix} (MIFIM):
\begin{equation}
\mathbf{F}_\varphi^{-1} = \mathbf{U} \mathbf{\Sigma}_\varphi \mathbf{V}^\top
\end{equation}
Based on the MIFIM, we organize the Modified natural gradient descent (Modified NGD) as 
\begin{equation}\label{MNGD}
\begin{aligned}
\tilde{\nabla}_{\theta} \mathcal{L} =& \mathbf{F}_\varphi^{-1} \nabla_\theta \mathcal{L},\\
\partial_t \theta_t =& -\eta\tilde{\nabla}_{\theta} \mathcal{L}.
\end{aligned}
\end{equation}
where $\eta$ is the learning rate.
\par
In the following sections, we will firstly derive the analytical solution of the Modified NGD, then prove in NTK regime that training with Modified NGD with proper criterion results in lower generalization error than ordinary NGD.

\section{Main Results}
\subsection{Analytical solution of Modified NGD}
Let us at first state the main assumptions in this paper:
\begin{assum}\label{assumGaussian}
For a data point $x$ and a network function $f$, we assume the output conditional probability $\tilde{p}(y|f(x))$ is Gaussian:
\begin{equation}
\tilde{p}(y|f(x)) = \frac{1}{\sqrt{2\pi}\sigma_0} e^{\frac{(y-f(x)^\top(y-f(x)))}{2\sigma_0^2}}.
\end{equation}
\end{assum}

\begin{assum}\label{assumInfiniteWidth}
The width of the layers of the neural network tends to infinity, that is in network expression \ref{NN}:
\begin{equation}
n_l \to \infty, \; l = 1,2,\dots, L-1.
\end{equation}
And the output layer is linear.
\end{assum}

\begin{assum}\label{assumPD}
The neural tangent kernel is \textit{positive definite}, or equivalently, the following term is positive definite:
\begin{equation*}
\nabla_\theta f_{\theta_0}(\mathcal{X}) \nabla_\theta f_{\theta_0}(\mathcal{X})^\top
\end{equation*}
\end{assum}
Since the empirical Fisher $\tilde{\mathbf{F}}(\theta_t) \in \mathbb{R}^{P\times P}$ is given by
\begin{equation}
\begin{aligned}\label{EmpiricalFisher}
\tilde{\mathbf{F}}(\theta_t) &= \frac{1}{N} \mathbb{E}_{\tilde{p}(y|f(\mathcal{X}))}\left[\nabla_\theta \log \tilde{p}(y|f(\mathcal{X})) \nabla_\theta \log \tilde{p}(y|f(\mathcal{X}))^\top\right]\\
&= \frac{1}{N} \mathbb{E}_{\tilde{p}(y|f(\mathcal{X}))}\left[\nabla_\theta f(\mathcal{X})^\top \nabla_f \log \tilde{p}(y|f(\mathcal{X})) \nabla_f \log \tilde{p}(y|f(\mathcal{X}))^\top \nabla_\theta f(\mathcal{X})\right].
\end{aligned}
\end{equation}
Under Assumption \ref{assumGaussian}, the empirical Fisher \ref{EmpiricalFisher} can be writer as 
\begin{equation}
\begin{aligned}
\tilde{\mathbf{F}}(\theta_t) &= \frac{1}{N\sigma_0^2}\nabla_\theta f(\mathcal{X})^\top \nabla_\theta f(\mathcal{X}).
\end{aligned}
\end{equation}
Under Assumption \ref{assumInfiniteWidth}, the neural network has the linearization expression as:
\begin{equation}
f_{\theta_t}(x) = f_{\theta_0}(x) + \nabla_\theta f_{\theta_0}(x)(\theta_t - \theta_0).
\end{equation}
 Under the linearlization, the Jacobian matrix of $f_\theta$ remains constain. Therefore, in the infinite width neural network, the NTK and the Fisher are both constant during training. Denoting the Jacobian matrix of $f_{\theta_t}$ evaluated on data points $\mathcal{X}$ at $\theta_t$ as $\mathbf{J}_t(\mathcal{X})$, and abbrevating $\mathbf{J}_0(\mathcal{X})$ for $\mathbf{J}$ unless otherwise specified. Moreover, as the width of neural network tends to infinite, with He initialization \cite{he2015delving}, the NTK is almost positive definite. Therefore, we can apply SVD to Jacobian matrix $\mathbf{J}$:
\begin{equation}
\mathbf{J} = \mathbf{U} \begin{pmatrix}
\mathbf{\Lambda} & \mathbf{0}_{N, P-N}
\end{pmatrix} \mathbf{V}^T
\end{equation}
where $\mathbf{U} \in \mathbb{R}^{N\times N}$, $\mathbf{V} \in \mathbb{R}^{P\times P}$ are both orthogonal matrices, that is $\mathbf{U}\mathbf{U}^\top = \mathbf{U}^\top \mathbf{U} = \mathbf{I}_N, \; \mathbf{V}\mathbf{V}^\top = \mathbf{V}^\top \mathbf{V} = \mathbf{I}_P$, and $\mathbf{\Lambda} = diag(\lambda_1, \dots, \lambda_N)$ with $\lambda_i \ne 0, \; i=1, \dots, N$ and $\lambda_1^2 \ge \dots \ge \lambda_N^2 > 0$. Thus, we have
\begin{gather}
\mathbf{K}_t(x,x')= \mathbf{K}_0(x,x') = \mathbf{J}(x) \mathbf{J}(x')^\top = \mathbf{U} \mathbf{\Lambda}^2 \mathbf{U}^\top,\\
\tilde{\mathbf{F}}(\theta_t) = \tilde{\mathbf{F}}(\theta_0) = \frac{1}{N\sigma_0^2}\mathbf{J}^\top \mathbf{J} = \frac{1}{n\sigma_0^2} \mathbf{V} \begin{pmatrix}
\mathbf{\Lambda}^2 & \mathbf{0}\\
\mathbf{0} & \mathbf{0}
\end{pmatrix} \mathbf{V}^\top. 
\end{gather}
The modification operation on $\mathbf{\Lambda}^2$ can be written as
\begin{equation}
\mathbf{\Lambda}^{-2}_\varphi \triangleq \varphi\left(\mathbf{\Lambda}^{-2}\right) = diag\left(\varphi\left(\frac{1}{\lambda_1^2}\right), \dots, \varphi\left(\frac{1}{\lambda_N^2}\right)\right) = \mathbf{\Lambda}^{-2} \left(\mathbf{I}_\varphi \oplus \mathbf{0}\right)
\end{equation}
where $\mathbf{I}_\varphi$ represents the positions being preserved, while $\mathbf{0}$ represents the positions being cut, and the $\oplus$ operator is the direct sum operator.
\par
Thus the empirical MIFIM can be wriiten as
\begin{equation}
\tilde{\mathbf{F}}_\varphi^\dagger = N\sigma_0^2 \mathbf{V} \begin{pmatrix}
\varphi\left(\mathbf{\Lambda}^{-2} \right) & \mathbf{0}\\
\mathbf{0} & \mathbf{0}
\end{pmatrix} \mathbf{V}^\top. = N\sigma_0^2 \mathbf{V} \begin{pmatrix}
\mathbf{\Lambda}^{-2} \left(\mathbf{I}_\varphi \oplus \mathbf{0}\right) & \mathbf{0}\\
\mathbf{0} & \mathbf{0}
\end{pmatrix} \mathbf{V}^\top
\end{equation}
Then we can derive the analytical solution of Modified NGD with training set $\mathcal{X}$ and $\mathcal{Y}$.
\par
In the following sections, unless otherwise specified, we abbreviate the empirical MIFIM at $\theta_0$ for $\mathbf{F}_\varphi^{\dagger}$.

\begin{thm}\label{CNGLinearTest}
Under Assumptions \ref{assumGaussian}, \ref{assumInfiniteWidth}, \ref{assumPD} and with $l_2$ loss, the solution of Modified NGD \ref{MNGD} trained on $\mathcal{X}$ and $\mathcal{Y}$ for time $T$ has prediction $f_{\theta_T}(x)$ on the test point $x \sim p_{data}(x)$, which can be expressed analytically as:
\begin{equation}
f_{\theta_t}(x) = f_{\theta_0}(x) -\left(1 - e^{-\eta N\sigma_0^2 t}\right) \mathbf{J}(x)  \mathbf{V} \begin{pmatrix}
\mathbf{\Lambda}^{-1} \left(\mathbf{I}_\varphi \oplus \mathbf{0}\right) \\
\mathbf{0}
\end{pmatrix} \mathbf{U}^\top\left(f_{\theta_0}(\mathcal{X}) - \mathcal{Y}\right).
\end{equation}
\end{thm}

\begin{proof}
Firstly, we derivate the solution of Modified NGD on training set $(\mathcal{X}, \mathcal{Y})$.
\par
Recall the dynamics \ref{MNGD}, the training dynamics of Modified NGD in function space on training set can be write as:
\begin{equation}\label{MNGD1}
\frac{\partial f_{\theta_t}(\mathcal{X})}{\partial t} = \frac{\partial f_{\theta_t}(\mathcal{X})}{\partial \theta_t} \frac{\partial \theta_t(\mathcal{X})}{\partial t} = -\eta \mathbf{J}\mathbf{F}_\varphi^\dagger \mathbf{J}^\top \left(f_{\theta_t}(\mathcal{X}) - \mathcal{Y}\right),
\end{equation}
Since
\begin{equation}\label{MNGD2}
\begin{aligned}
\mathbf{J}\mathbf{F}_\varphi^\dagger \mathbf{J}^\top = N\sigma_0^2 \mathbf{U} \left(\mathbf{I}_\varphi \oplus \mathbf{0}\right) \mathbf{U}^\top
\end{aligned}
\end{equation}
we can analytically solve this ODE by
\begin{equation}\label{SolutionTraining}
\begin{aligned}
f_{\theta_t}(\mathcal{X}) &= \mathcal{Y} + e^{-\eta N\sigma_0^2 \mathbf{U} \left(\mathbf{I}_\varphi \oplus \mathbf{0}\right) \mathbf{U}^\top t}\left(f_{\theta_0}(\mathcal{X}) - \mathcal{Y}\right)\\
&= \mathcal{Y} + e^{-\eta N\sigma_0^2 t} \mathbf{U} \left(\mathbf{I}_\varphi \oplus \mathbf{0}\right) \mathbf{U}^\top\left(f_{\theta_0}(\mathcal{X}) - \mathcal{Y}\right), \quad \forall t \in [0,T].
\end{aligned}
\end{equation}
After that, let us foucs on the function dynamics on test point $x \sim p_{data}(x)$. Recall the expression of $f_{\theta_t}(\mathcal{X})$ in equation \ref{SolutionTraining}, we have
\begin{equation}
\begin{aligned}
\frac{\partial f_{\theta_t}(x)}{\partial t} &=  -\eta \mathbf{J}(x)\mathbf{F}_\varphi^\dagger \mathbf{J}^\top \left(f_{\theta_t}(\mathcal{X}) - \mathcal{Y}\right)\\
&= -\eta \mathbf{J}(x)\mathbf{F}_\varphi^\dagger \mathbf{J}^\top e^{-\eta N\sigma_0^2 t} \mathbf{U} \left(\mathbf{I}_\varphi \oplus \mathbf{0}\right) \mathbf{U}^\top\left(f_{\theta_0}(\mathcal{X}) - \mathcal{Y}\right)\\
&= -\eta N\sigma_0^2 e^{-\eta N\sigma_0^2 t} \mathbf{J}(x)  \mathbf{V} \begin{pmatrix}
\mathbf{\Lambda}^{-1} \left(\mathbf{I}_\varphi \oplus \mathbf{0}\right) \\
\mathbf{0}
\end{pmatrix} \mathbf{U}^\top\left(f_{\theta_0}(\mathcal{X}) - \mathcal{Y}\right).
\end{aligned}
\end{equation}
Integrad by $t$ in the two sides of this equation, we get
\begin{equation}
f_{\theta_t}(x) = f_{\theta_0}(x) -\left(1 - e^{-\eta N\sigma_0^2 t}\right) \mathbf{J}(x)  \mathbf{V} \begin{pmatrix}
\mathbf{\Lambda}^{-1} \left(\mathbf{I}_\varphi \oplus \mathbf{0}\right) \\
\mathbf{0}
\end{pmatrix} \mathbf{U}^\top\left(f_{\theta_0}(\mathcal{X}) - \mathcal{Y}\right).
\end{equation}
This solution holds for $\forall t \in [0,T]$. In particular, it holds for $t=T$, which concludes the proof.
\end{proof}
More detailed proof of Theorem \ref{CNGLinearTest} can be found in the \textbf{Supplementary Materials}.
\begin{remark}\label{directionRemark}
Recall \ref{MNGD1}, \ref{MNGD2} and that $\alpha(\mathcal{X}, \mathcal{Y}) \triangleq f_{\theta_0}(\mathcal{X}) - \mathcal{Y} = \nabla_f \mathcal{L}$ is gradient in function space. The training dynamics by Modified NGD can be regarded as the orthogonal sum of training dynamics in different directions:
\begin{equation}
\begin{aligned}
\frac{\partial f_{\theta_t}(\mathcal{X})}{\partial t} =& -\eta N\sigma_0^2 \mathbf{U} \left(\mathbf{I}_\varphi \oplus \mathbf{0}\right) \mathbf{U}^\top \alpha(\mathcal{X}, \mathcal{Y})\\
=& -\eta N\sigma_0^2 \sum\limits_{i=1}^N \lambda_i^2 \varphi\left(\frac{1}{\lambda_i^2}\right) (\alpha(\mathcal{X}, \mathcal{Y})^\top \mathbf{u}_i) \mathbf{u}_i
\end{aligned}
\end{equation}
Notice that $\mathbf{u}_i$ represents the eigenspace of NTK, therefore the training dynamics of modified NGD can be regarded as modifying the training directions in the eigenspace of NTK in function space.
\end{remark}
Theorem \ref{CNGLinearTest} gives the neural network function trained by Modified NGD algorithm for time $T$. As the convergence theory of NG algorithm \cite{bernacchia2018exact}, we claim that the network function trained by Modified NGD converges as $T \to \infty$.
\begin{coro}\label{SimplifiedCNG}
The network function trained by Modified NGD converges to $f_{\theta_\infty}(x)$ as $T \to \infty$,
\begin{equation}\label{fInfty}
f_{\theta_\infty}(x) = \lim\limits_{T \to \infty}f_{\theta_T}(x) = f_{\theta_0}(x) - \mathbf{J}(x)  \mathbf{V} \begin{pmatrix}
\mathbf{\Lambda}^{-1} \left(\mathbf{I}_\varphi \oplus \mathbf{0}\right) \\
\mathbf{0}
\end{pmatrix} \mathbf{U}^\top\left(f_{\theta_0}(\mathcal{X}) - \mathcal{Y}\right).
\end{equation}
\end{coro}
\par
Based on the solutions given by Theorem \ref{CNGLinearTest} and Corollary \ref{SimplifiedCNG}, we can analyze the generalization error of the network function trained by Modified NGD on training set. In the next subsection, we will derive the decomposition of generalization error obtained by the convergence network function.

\subsection{Generalization error bound}
Then for the convergence network function $f_{\theta_\infty}$ trained by Modified NGD, we can decompose the generalization error of it into two components, one stemming from training set and the other stemming from the distribution discrepancy between the training set and the true data.
\begin{thm}\label{GeneralizationThm}
Under the same assumptions as Theorem \ref{CNGLinearTest}, the expected risk of $f_{\theta_\infty}$ trained by CNG in Corollary \ref{SimplifiedCNG} can be decomposed into two parts, one of the risk on training set, one of the risk on the distribution discrepancy between training set and true data:
\begin{equation}\label{GeneError}
\begin{gathered}
\mathcal{R}(f_{\theta_\infty}) = \mathcal{R}_1 + \mathcal{R}_2,\\
 \mathcal{R}_1 = \frac{1}{N} \alpha(\mathcal{X}, \mathcal{Y})^\top \mathbf{U} \left( \mathbf{I}_N - \mathbf{\Lambda}^2 \mathbf{\Lambda}_\varphi^{-2}\right) \mathbf{U}^\top\alpha(\mathcal{X}, \mathcal{Y}),\\
\mathcal{R}_2 = \mathbf{B} - 2\mathbf{L}\begin{pmatrix}
\mathbf{\Lambda} \mathbf{\Lambda}_\varphi^{-2} \\
\mathbf{0}
\end{pmatrix} \mathbf{U}^\top\alpha(\mathcal{X}, \mathcal{Y})+ \alpha(\mathcal{X}, \mathcal{Y})^\top \mathbf{U} \begin{pmatrix}
\mathbf{\Lambda} \mathbf{\Lambda}_\varphi^{-2} & \mathbf{0}\end{pmatrix}
\mathbf{Q}
\begin{pmatrix}
\mathbf{\Lambda} \mathbf{\Lambda}_\varphi^{-2} \\
\mathbf{0}
\end{pmatrix} \mathbf{U}^\top \alpha(\mathcal{X},\mathcal{Y}).
\end{gathered}
\end{equation}
where $\alpha(\mathcal{X},\mathcal{Y}) = f_{\theta_0}(\mathcal{X}) - \mathcal{Y}$, and
\begin{equation}
\begin{aligned}
\mathbf{B} =& \left(\mathbb{E}_{x,y}\left[\alpha(x,y)^2\right] -\frac{1}{N}\alpha(\mathcal{X}, \mathcal{Y})^\top \alpha(\mathcal{X}, \mathcal{Y}) \right),\\
\mathbf{L} =& \left(\mathbb{E}_{x,y}\left[\alpha(x,y) \mathbf{J}(x) \right] - \frac{1}{N} \alpha(\mathcal{X}, \mathcal{Y})^\top \mathbf{J} \right) \mathbf{V},\\
\mathbf{Q} =& \mathbf{V}^\top \left(\mathbb{E}_{x}\left[\mathbf{J}(x)^\top \mathbf{J}(x)\right] - \frac{1}{N}\mathbf{J}^\top \mathbf{J}\right) \mathbf{V}.
\end{aligned}
\end{equation}
\end{thm}
\begin{proof}
Recall the definition of expected risk \ref{expectedRisk} and the expression of $f_\infty$ in \ref{fInfty}, we have
\begin{equation}
\begin{aligned}
\mathcal{R}(f_{\theta_\infty})&= \mathbb{E}_{x,y}\left[\left(\alpha(x,y)\right)^2\right]- \underbrace{2\mathbb{E}_{x,y}\left[\alpha(x,y) \mathbf{J}(x)  \mathbf{V} \begin{pmatrix}
\mathbf{\Lambda}^{-1} \left(\mathbf{I}_\varphi \oplus \mathbf{0}\right) \\
\mathbf{0}
\end{pmatrix} \mathbf{U}^\top\alpha(\mathcal{X},\mathcal{Y})\right]}_{T_1}\\
&+ \underbrace{\mathbb{E}_{x}\left[\left(\mathbf{J}(x)  \mathbf{V} \begin{pmatrix}
\mathbf{\Lambda}^{-1} \left(\mathbf{I}_\varphi \oplus \mathbf{0}\right) \\
\mathbf{0}
\end{pmatrix} \mathbf{U}^\top\alpha(\mathcal{X},\mathcal{Y})\right)^2\right]}_{T_2}.
\end{aligned}
\end{equation}
where $\alpha(x,y) = f_{\theta_0}(x) - y$. Since
\begin{equation}
\begin{aligned}
\mathbb{E}_{x,y}\left[\alpha(x,y)^2\right] &= \frac{1}{N}\alpha(\mathcal{X}, \mathcal{Y})^\top \alpha(\mathcal{X}, \mathcal{Y}) + \left(\mathbb{E}_{x,y}\left[\alpha(x,y)^2\right] -\frac{1}{N}\alpha(\mathcal{X}, \mathcal{Y})^\top \alpha(\mathcal{X}, \mathcal{Y}) \right)\\
&\triangleq \frac{1}{N}\alpha(\mathcal{X}, \mathcal{Y})^\top \alpha(\mathcal{X}, \mathcal{Y}) + \mathbf{B}.
\end{aligned}
\end{equation}
where $\mathbf{B} = \left(\mathbb{E}_{x,y}\left[\alpha(x,y)^2\right] -\frac{1}{N}\alpha(\mathcal{X}, \mathcal{Y})^\top \alpha(\mathcal{X}, \mathcal{Y}) \right)$.
\par
For $T_1$ and $T_2$, by similar decomposition, we have
\begin{equation}
\begin{aligned}
T_1 =& \frac{2}{N} \alpha(\mathcal{X}, \mathcal{Y})^\top \mathbf{U}\left(\mathbf{I}_\varphi \oplus \mathbf{0}\right) \mathbf{U}^\top\alpha(\mathcal{X}, \mathcal{Y})
+ 2\mathbf{L}\begin{pmatrix}
\mathbf{\Lambda}^{-1} \left(\mathbf{I}_\varphi \oplus \mathbf{0}\right) \\
\mathbf{0}
\end{pmatrix} \mathbf{U}^\top\alpha(\mathcal{X}, \mathcal{Y}),
\end{aligned}
\end{equation}
where $\mathbf{L} = \left(\mathbb{E}_{x,y}\left[\alpha(x,y) \mathbf{J}(x) \right] - \frac{1}{N} \alpha(\mathcal{X}, \mathcal{Y})^\top \mathbf{J} \right) \mathbf{V}$. 
\begin{equation}
\begin{aligned}
T_2 &= \frac{1}{N}\alpha(\mathcal{X}, \mathcal{Y})^\top \mathbf{U} \left(\mathbf{I}_\varphi \oplus \mathbf{0}\right) \mathbf{U}^\top \alpha(\mathcal{X}, \mathcal{Y})\\
&+ \alpha(\mathcal{X}, \mathcal{Y})^\top \mathbf{U} \begin{pmatrix}
\mathbf{\Lambda}^{-1}\left(\mathbf{I}_\varphi \oplus \mathbf{0}\right) & \mathbf{0}\end{pmatrix}
\mathbf{Q}
\begin{pmatrix}
\mathbf{\Lambda}^{-1} \left(\mathbf{I}_\varphi \oplus \mathbf{0}\right) \\
\mathbf{0}
\end{pmatrix} \mathbf{U}^\top\alpha(\mathcal{X}, \mathcal{Y}),
\end{aligned}
\end{equation}
where $\mathbf{Q} = \mathbf{V}^\top \left(\mathbb{E}_{x}\left[\mathbf{J}(x)^\top \mathbf{J}(x)\right] - \frac{1}{N}\mathbf{J}^\top \mathbf{J}\right) \mathbf{V}$.
\par
Therefore, the generaliztion error can be split to two parts:
\begin{equation}
\mathcal{R}(f_{\theta_\infty})= \mathcal{R}_1 + \mathcal{R}_2
\end{equation}
with
\begin{equation*}
\begin{aligned}
\mathcal{R}_1 &= \frac{1}{N}\alpha(\mathcal{X}, \mathcal{Y})^\top \alpha(\mathcal{X}, \mathcal{Y}) - \frac{1}{N} \alpha(\mathcal{X}, \mathcal{Y})^\top \mathbf{U}\left(\mathbf{I}_\varphi \oplus \mathbf{0}\right) \mathbf{U}^\top\alpha(\mathcal{X}, \mathcal{Y})\\
&= \frac{1}{N} \alpha(\mathcal{X}, \mathcal{Y})^\top \mathbf{U} \left( \mathbf{I}_N - \mathbf{\Lambda}^2 \mathbf{\Lambda}_\varphi^{-2}\right) \mathbf{U}^\top\alpha(\mathcal{X}, \mathcal{Y}),
\end{aligned}
\end{equation*}
and
\begin{equation*}
\begin{aligned}
\mathcal{R}_2 &= \mathbf{B} - 2\mathbf{L}\begin{pmatrix}
\mathbf{\Lambda} \mathbf{\Lambda}_\varphi^{-2} \\
\mathbf{0}
\end{pmatrix} \mathbf{U}^\top\alpha(\mathcal{X}, \mathcal{Y})+ \alpha(\mathcal{X}, \mathcal{Y})^\top \mathbf{U} \begin{pmatrix}
\mathbf{\Lambda} \mathbf{\Lambda}_\varphi^{-2} & \mathbf{0}\end{pmatrix}
\mathbf{Q}
\begin{pmatrix}
\mathbf{\Lambda} \mathbf{\Lambda}_\varphi^{-2} \\
\mathbf{0}
\end{pmatrix} \mathbf{U}^\top \alpha(\mathcal{X},\mathcal{Y}).
\end{aligned}
\end{equation*}
\end{proof}
More detailed proof can be found in the \textbf{Supplementary Materials}.
\par
\begin{remark}\label{errorRemark}
In expression of the decomposition of expected risk, $\mathcal{R}_1$ represents the error obtained from the training set, while $\mathcal{R}_2$ represents the error caused by the distribution discrepancy between the training set and the true data, where $\mathbf{Q}, \mathbf{L}, \mathbf{B}$ measure the distribution discrepancy between the training set and the true data. 
\end{remark}
Based on the above observations, under specific condition, we can derive a criterion to decide the directions to be modified to reduce the total generalization error. Our results are stated in the following.

\begin{coro}
For the result in Theorem \ref{GeneralizationThm}, if the training set were drawn i.i.d. from the true data distribution and the training set is large enough, the risk above can be written in orthogonal form:
\begin{equation}
\begin{aligned}
\mathcal{R}_1 =& \frac{1}{N} \sum\limits_{i=1}^N \left(1 - \lambda_i^2 \varphi\left(\frac{1}{\lambda_i^2}\right)\right) \left(\alpha(\mathcal{X},\mathcal{Y})^\top u_i\right)^2,\\
\mathcal{R}_2 \approx& \sum\limits_{i=1}^N \left(q_i \lambda_i^2 \varphi\left(\frac{1}{\lambda_i^2}\right)^2 \left(\alpha(\mathcal{X},\mathcal{Y})^\top u_i\right)^2 -2 l_i \lambda_i \varphi\left(\frac{1}{\lambda_i^2}\right) \alpha(\mathcal{X},\mathcal{Y})^\top u_i + b_i \right).
\end{aligned}
\end{equation}
and
\begin{equation}
\begin{aligned}
\mathcal{R}(f_\infty) =& \mathcal{R}_1 + \mathcal{R}_2.
\end{aligned}
\end{equation}
Thus with the criteirion $c\left(\frac{1}{\lambda_i^2}\right)$ defined as
\begin{equation}\label{criterion}
q_i \frac{1}{\lambda_i^2} - \frac{2l_i}{\alpha(\mathcal{X},\mathcal{Y})^\top u_i}\frac{1}{\lambda_i} -\frac{1}{N} > 0,
\end{equation}
the Modified NGD can reduce the generalization error.
\end{coro}
\begin{proof}
The approximation is derive from the law of large number, and the criterion can be obtain directly from setting the error without modification greater than with modification in each eigenspace. More detailed proof can be found in the \textbf{Supplementary Materials}.
\end{proof}

\begin{remark}
From the expressions of error terms $\mathcal{R}_1$ and $\mathcal{R}_2$, we can observe that ordinary NGD walks to the interpolation function of training set while deviates from the true model in function space. Recall the discussion in Remark \ref{errorRemark}, with cutting proper eigenvalues of the empirical MIFIM, the training dynamics in function space can stop training on the eigenspace which of large discrepancy between the model of training set and the true model, while remaining training on the eigenspace with small discrepancy, therefore walks to a point with better generalization performance in function space.
\end{remark}

\section{Numerical Experiments}
This section aims to illustrate our theoretical results of Modified NGD, that is, based on the theoretical criterion of modification, the Modified NGD can reduce the generalization error compared with ordinary NGD and NGD with modification on other directions.
\par
Due to the high dimension of Fisher, all of our experiments are implemented on a two layers MLP (Multi-Layer Perceptron) with synthetic data\footnote{All codes, data and results can be find at \url{https://github.com/21veu/modified_NGD}. More details can be found in the \textbf{Supplementary Materials}.}. However, with the dicussions of the discrepancy bounds of NTK regime and general neural network \cite{jacot2018neural, arora2019exact, rudner2019natural}, our theoretical and numerical results can be generalized to general DNN.
\par
\textbf{Setup} We firstly draw samples uniformly from interval $[0,1)$, then split the samples to training set with 256 samples, validation set with 64 samples and test set with 64 samples, and apply perturbation to the training set:
\begin{equation}
x \to xe^{-\frac{(1-x)^2}{\sigma^2}}
\end{equation}
with different perturbation factor $\sigma^2$.
\par
For a function approximation problem:
\begin{equation}
f^\star(x) = \cos x \sin x.
\end{equation}
We use a two layers MLP with $2^{12}$ neurons with He initialization \cite{he2015delving} to train on the training set. We perform two optimization algorithms: modified NGD and NGD with all the other settings being same. Modified NGD uses validation set for the true distribution computation in the criterion \ref{criterion} to decide the directions to be modified. The initial learning rate is set as 0.1 with learning rate half decay and train for 500 epochs. We run each experiments for 20 random seed. and the results are reported on the average of different random seeds. We implements the numerical experiments for different degrees of perturbation with the mean of perturbed data changing roughly equally, thus we choose the perturbation factors $\sigma^2$ to be: 10, 5, 1.5 and 1.

\textbf{Results} As shown in Fig.\ref{sigmas}, with different degrees of perturbation on training data, the Modified NGD is stable and apparently of smaller generalization error than NGD. In the plots, a line represents the mean on random seeds and the envelope around it reflects 0.3 times standard deviation. 
\par 
Fig \ref{diffloss} illustrates the trendency of the difference of NGD and Modified NGD at the convergence point. As $\sigma^2$ decreases, the degree of perturbation on training data increases, then the generalization error of NGD increase, the performance of NGD gets worse than Modified NGD. 
\par
To varify our theoretical results of generalization decomposition, we implement a comparative experiment to illustrate that the criterion \ref{criterion} derived from our decomposition is more effective than other criterions. Inspired by many existing results such as \cite{bartlett2002rademacher, keskar2016large}, a view believes that cutting small eigenvalues benefits generalization. We trained on the training set perturbed with perturbation factor $\sigma^2 = 1$ by Modified NGD and NGD cut as many as eigenvalues but small ones. The test loss of these two algorithms are shown in Fig. \ref{cutsmall}. The experimental results demonstrate that cutting only small eigenvalues is effective to reduce the generalization error, but not as good as our criterion \ref{criterion}, which verified our theoretical results and corrected the conventional view.

\begin{figure}[h]
	\begin{minipage}{0.24\linewidth}
		\centering
  		\includegraphics[width=\textwidth]{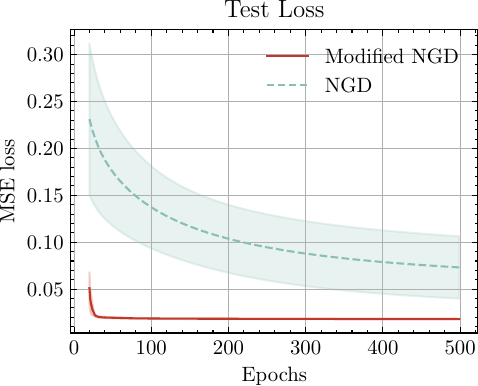}
  		\subcaption{$\sigma^2=10$}
  		\label{testloss10}
	\end{minipage}
	\begin{minipage}{0.24\linewidth}
  		\includegraphics[width=\textwidth]{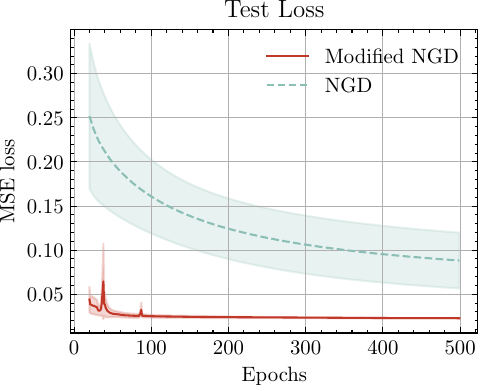}
  		\subcaption{$\sigma^2=5$}
  		\label{testloss5}
	\end{minipage}
	\begin{minipage}{0.24\linewidth}
  		\includegraphics[width=\textwidth]{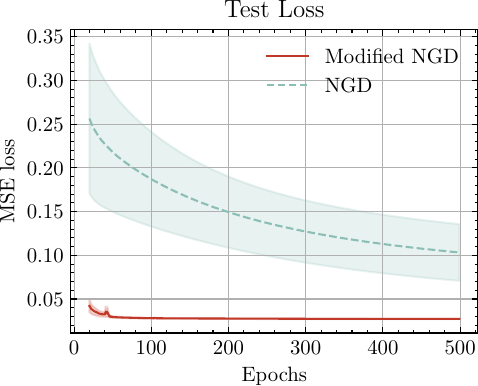}
  		\subcaption{$\sigma^2=1.5$}
  		\label{testloss1.5}
  		
	\end{minipage}
	\begin{minipage}{0.24\linewidth}
  		\includegraphics[width=\textwidth]{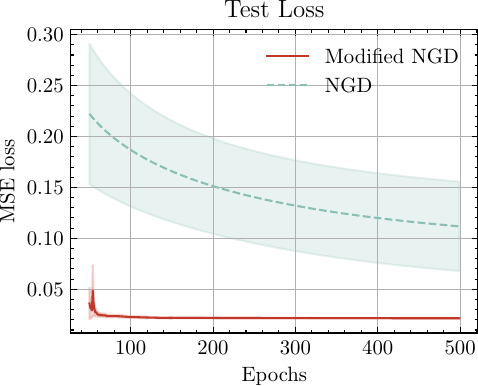}
  		\subcaption{$\sigma^2=1$}
  		\label{testloss1}
	\end{minipage}
	\caption{The test loss of NGD and Modified NGD with different degrees of perturbation during 500 epochs. \ref{testloss10} is the test loss results trained on training data perturbed with the perturbation factor $\sigma^2=10$; \ref{testloss5} with the perturbation factor $\sigma^2=5$; \ref{testloss1.5} with the perturbation factor $\sigma^2=1.5$; \ref{testloss1} with the perturbation factor $\sigma^2=1$. }
	\label{sigmas}
\end{figure}

\begin{figure}[h]
\begin{minipage}{0.29\linewidth}
		\centering
  		\includegraphics[width=\textwidth]{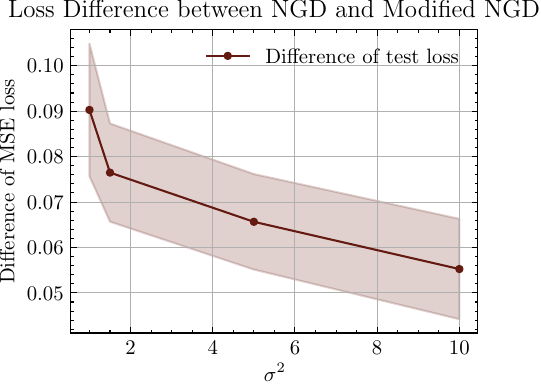}
  		\caption{The average difference of test loss of NGD and Modified NGD in the last 10 epochs with respect to different perturbation factor $\sigma^2$.}
  		\label{diffloss}
\end{minipage}\quad
\begin{minipage}{0.60\linewidth}
		\centering
		\includegraphics[width=0.45\textwidth]{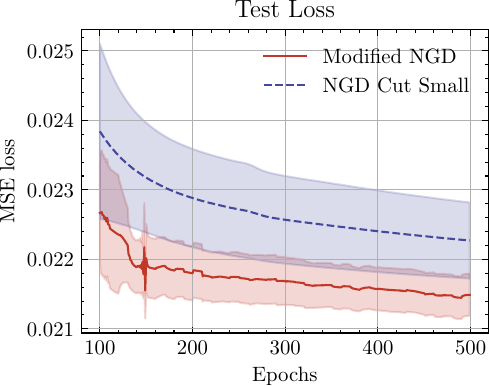}\quad
  		\includegraphics[width=0.5\textwidth]{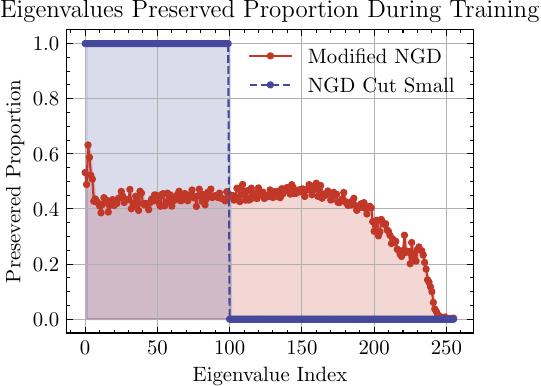}
  		\caption{The left subfigure shows the test loss of NGD that cut eigenvalues as many as Modified NGD but small ones and Modified NGD with perturbation factor $\sigma^2=1$ during 500 epochs; the right subfigure shows the proportion of each eigenvalue being preserved during 500 epochs}
		\label{cutsmall}
\end{minipage}
	
\end{figure}

As shown in Fig.\ref{diffloss} we can observe from the difference of the losses between NGD and modified NGD in the last epochs that with the $\sigma^2$ increasing, the difference of loss decreases. The increasing of $\sigma^2$ indicates the perturbation on training set is decreasing, thus the difference between training set and test set is getting smaller, the training direction of ordinary NGD is getting close to modified NGD.

\section{Insights on existing generalization enhancing algorithms}\label{insights}
Since the machine learning achieved good performance on a lot of tasks, several algorithms aiming for enhancing the generalization performance based on existing resources were proposed. Our results can shed light on why these algorithms work theoretically. In the following, we give a few examples.
\par
\textbf{Cross domain generalization} Zeyi Huang et al. \cite{huang2020self} proposed a intuitive algorithm to enhance the performance of crosss domain generalization by cut the largest components of the gradient. And we can demonstrate that this modification is equivalent to cut the largest eigenvalues of empirical Fisher informantion matrix, therefore modified the training directions of neural network in function space. The proof of this demonstration can be found in the \textbf{Supplementary Materials}.
\par
\textbf{Self distillation} Self distillation is a post-training method. Mobahi et al. \cite{mobahi2020self} shows that self distillation amplifies regularization effect at each distillation round, which make the eigenvalues of the Gram matrix of the kernel of the regularizer evolve. And after several distillation rounds, the new corresponding kernel's Gram matrix possesses smaller eigenvalues, thus enhances the generalization performance. In our framework, the eigenvalues of Gram matrix indicate the training dynamics in eigenspace, which have the same effect as Fisher information matrix. Therefore, self distillation employs a mild modification on training directions in function space introduced by the kernel. More detailed demonstration can be found in the \textbf{Supplementary Materials}.
\par
\textbf{Small batch training and flat minima} Keskar et al. \cite{keskar2016large} proposed a metric to measure the sharpness of local minima, which is related to the the eigenvalues of the Hessian around the local minima. As discussed in \cite{martens2020new} and \cite{liu2020linearity}, in overparameterized neural network, the Hessian is an approximation for Fisher information matrix, which indicates that modifying the training direction by modifying the eigenvalues of Fisher information matrix will change the sharpness of the convergence point in the function space. Thus with our results, it can be proved that flat minima corresponds to convergence point with lower 
generalization error. More detailed demonstration can be found in the \textbf{Supplementary Materials}.

\section{Conclusion}
We firstly presented an Modified NGD framework and proceed to derive an analytical expression for the function trained by this Modified NGD. Based on this solution, we explicitly computed the generalization error of the learned neural network function and decomposed it into two components: the errors arising from training set and stemming from the distribution discrepancy between the training set and the true data. Moreover, under the i.i.d. condition of dataset, we decomposed the error attributed to different eigenspace of NTK in function space and we proposed a criterion to decide the directions to be modified. We established theoretical results and implemented numerical experiments to verify that modifying the training direction of the neural network in function space leads to a reduction in the total generalization error. Furthermore, We demonstrate that this theoretical framework is capable to explain many existing results of generalization enhancing methods.

\bibliography{Overparameterization}
\bibliographystyle{ieeetr}

\clearpage

\section*{Supplementary Material}
\section*{The derivation of Empirical Fisher information matrix}
Recall Assumption \ref{assumGaussian}:
\begin{assum}\label{assumGaussian}
For a data point $x$ and a network function $f$, we assume the output conditional probability $\tilde{p}(y|f(x))$ is Gaussian:
\begin{equation}
\tilde{p}(y|f(x)) = \frac{1}{(\sqrt{2\pi}\sigma_0)^{n_{out}}} e^{\frac{(y-f(x)^\top(y-f(x)))}{2\sigma_0^2}}.
\end{equation}
\end{assum}
Since the empirical Fisher $\tilde{F}(\theta_t) \in \mathbb{R}^{P\times P}$ is given by
\begin{equation}
\begin{aligned}\label{EmpiricalFisher}
\tilde{\mathbf{F}}(\theta_t) &= \frac{1}{N} \mathbb{E}_{\tilde{p}(y|f(\mathcal{X}))}\left[\nabla_\theta \log \tilde{p}(y|f(\mathcal{X})) \nabla_\theta \log \tilde{p}(y|f(\mathcal{X}))^\top\right]\\
&= \frac{1}{N} \mathbb{E}_{\tilde{p}(y|f(\mathcal{X}))}\left[\nabla_\theta f(\mathcal{X})^\top \nabla_f \log \tilde{p}(y|f(\mathcal{X})) \nabla_f \log \tilde{p}(y|f(\mathcal{X}))^\top \nabla_\theta f(\mathcal{X})\right].
\end{aligned}
\end{equation}
If we assume the output probability is Gaussian,
\begin{equation}
\nabla_f \log \tilde{p}(y|f(\mathcal{X})) = \nabla_f \left(\frac{(y-f(\mathcal{X})^\top(y-f(\mathcal{X})))}{2\sigma_0^2}\right) = \frac{y-f(\mathcal{X})}{\sigma_0^2},
\end{equation}
Then, the empirical Fisher \ref{EmpiricalFisher} can be writer as 
\begin{equation}
\begin{aligned}
\tilde{\mathbf{F}}(\theta_t) &= \frac{1}{N\sigma_0^4} \mathbb{E}_{\tilde{p}(y|f(\mathcal{X}))}\left[\nabla_\theta f(\mathcal{X})^\top (y-f(\mathcal{X})) (y-f(\mathcal{X}))^\top \nabla_\theta f(\mathcal{X}) \right]\\
&= \frac{1}{N\sigma_0^4}\nabla_\theta f(\mathcal{X})^\top \mathbb{E}_{\tilde{p}(y|f(\mathcal{X}))}\left[ (y-f(\mathcal{X})) (y-f(\mathcal{X}))^\top \right] \nabla_\theta f(\mathcal{X})\\
&= \frac{1}{N\sigma_0^4}\nabla_\theta f(\mathcal{X})^\top \sigma_0^2 I \nabla_\theta f(\mathcal{X})\\
&= \frac{1}{N\sigma_0^2}\nabla_\theta f(\mathcal{X})^\top \nabla_\theta f(\mathcal{X}).
\end{aligned}
\end{equation}

\section*{Proof of Theorem 1}
\begin{thm}\label{CNGLinearTest}
Under Assumptions \ref{assumGaussian}, 2, 3 and with $l_2$ loss, the solution of Modified NGD trained on $\mathcal{X}$ and $\mathcal{Y}$ for time $T$ has prediction $f_{\theta_T}(x)$ on the test point $x \sim p_{data}(x)$, which can be expressed analytically as:
\begin{equation}
f_{\theta_t}(x) = f_{\theta_0}(x) -\left(1 - e^{-\eta N\sigma_0^2 t}\right) \mathbf{J}(x)  \mathbf{V} \begin{pmatrix}
\mathbf{\Lambda}^{-1} \left(\mathbf{I}_\varphi \oplus \mathbf{0}\right) \\
\mathbf{0}
\end{pmatrix} \mathbf{U}^\top\left(f_{\theta_0}(\mathcal{X}) - \mathcal{Y}\right).
\end{equation}
\end{thm}

\begin{proof}
Firstly, we derivate the solution of Modified NGD on training set $(\mathcal{X}, \mathcal{Y})$.
\par
The training dynamics of Modified NGD in function space on training set can be write as:
\begin{equation}\label{MNGD1}
\frac{\partial f_{\theta_t}(\mathcal{X})}{\partial t} = \frac{\partial f_{\theta_t}(\mathcal{X})}{\partial \theta_t} \frac{\partial \theta_t(\mathcal{X})}{\partial t} = -\eta \mathbf{J}\mathbf{F}_\varphi^\dagger \mathbf{J}^\top \left(f_{\theta_t}(\mathcal{X}) - \mathcal{Y}\right),
\end{equation}
Since
\begin{equation}\label{MNGD2}
\begin{aligned}
\mathbf{J}\mathbf{F}_\varphi^\dagger \mathbf{J}^\top = N\sigma_0^2 \mathbf{U} \left(\mathbf{I}_\varphi \oplus \mathbf{0}\right) \mathbf{U}^\top
\end{aligned}
\end{equation}
\begin{equation}
\frac{\partial f_{\theta_t}(\mathcal{X})}{\partial t} = -\eta N\sigma_0^2 \mathbf{U} \left(\mathbf{I}_\varphi \oplus \mathbf{0}\right) \mathbf{U}^\top \left(f_{\theta_t}(\mathcal{X}) - \mathcal{Y}\right),
\end{equation}
we can analytically solve this ODE by
\begin{equation}\label{SolutionTraining}
\begin{aligned}
f_{\theta_t}(\mathcal{X}) &= \mathcal{Y} + e^{-\eta N\sigma_0^2 \mathbf{U} \left(\mathbf{I}_\varphi \oplus \mathbf{0}\right) \mathbf{U}^\top t}\left(f_{\theta_0}(\mathcal{X}) - \mathcal{Y}\right)\\
&= \mathcal{Y} + e^{-\eta N\sigma_0^2 t} \mathbf{U} \left(\mathbf{I}_\varphi \oplus \mathbf{0}\right) \mathbf{U}^\top\left(f_{\theta_0}(\mathcal{X}) - \mathcal{Y}\right), \quad \forall t \in [0,T].
\end{aligned}
\end{equation}
After that, let us foucs on the function dynamics on test point $x \sim p_{data}(x)$. Recall the expression of $f_{\theta_t}(\mathcal{X})$ in equation \ref{SolutionTraining}, we have
\begin{equation}
\begin{aligned}
\frac{\partial f_{\theta_t}(x)}{\partial t} &=  -\eta \mathbf{J}(x)\mathbf{F}_\varphi^\dagger \mathbf{J}^\top \left(f_{\theta_t}(\mathcal{X}) - \mathcal{Y}\right)\\
&= -\eta \mathbf{J}(x)\mathbf{F}_\varphi^\dagger \mathbf{J}^\top e^{-\eta N\sigma_0^2 t} \mathbf{U} \left(\mathbf{I}_\varphi \oplus \mathbf{0}\right) \mathbf{U}^\top\left(f_{\theta_0}(\mathcal{X}) - \mathcal{Y}\right)\\
&= -\eta N\sigma_0^2 e^{-\eta N\sigma_0^2 t} \mathbf{J}(x)  \mathbf{V} \begin{pmatrix}
\mathbf{\Lambda}^{-1} \left(\mathbf{I}_\varphi \oplus \mathbf{0}\right) \\
\mathbf{0}
\end{pmatrix} \mathbf{U}^\top\left(f_{\theta_0}(\mathcal{X}) - \mathcal{Y}\right).
\end{aligned}
\end{equation}
Integrad by $t$ in the two sides of this equation, we get
\begin{equation}
f_{\theta_t}(x) = f_{\theta_0}(x) -\left(1 - e^{-\eta N\sigma_0^2 t}\right) \mathbf{J}(x)  \mathbf{V} \begin{pmatrix}
\mathbf{\Lambda}^{-1} \left(\mathbf{I}_\varphi \oplus \mathbf{0}\right) \\
\mathbf{0}
\end{pmatrix} \mathbf{U}^\top\left(f_{\theta_0}(\mathcal{X}) - \mathcal{Y}\right).
\end{equation}
This solution holds for $\forall t \in [0,T]$. In particular, it holds for $t=T$, which concludes the proof.
\end{proof}

\section*{Proof of Theorem 2}

\begin{thm}\label{GeneralizationThm}
Under the same assumptions as Theorem \ref{CNGLinearTest}, the expected risk of $f_{\theta_\infty}$ trained by Modified NGD in Corollary 1 can be decomposed into two parts, one of the risk on training set, one of the risk on the distribution discrepancy between training set and true data:
\begin{equation}\label{GeneError}
\begin{gathered}  
\mathcal{R}(f_{\theta_\infty}) = \mathcal{R}_1 + \mathcal{R}_2,\\
 \mathcal{R}_1 = \frac{1}{N} \alpha(\mathcal{X}, \mathcal{Y})^\top \mathbf{U} \left( \mathbf{I}_N - \mathbf{\Lambda}^2 \mathbf{\Lambda}_\varphi^{-2}\right) \mathbf{U}^\top\alpha(\mathcal{X}, \mathcal{Y}),\\
\mathcal{R}_2 = \mathbf{B} - 2\mathbf{L}\begin{pmatrix}
\mathbf{\Lambda} \mathbf{\Lambda}_\varphi^{-2} \\
\mathbf{0}
\end{pmatrix} \mathbf{U}^\top\alpha(\mathcal{X}, \mathcal{Y})+ \alpha(\mathcal{X}, \mathcal{Y})^\top \mathbf{U} \begin{pmatrix}
\mathbf{\Lambda} \mathbf{\Lambda}_\varphi^{-2} & \mathbf{0}\end{pmatrix}
\mathbf{Q}
\begin{pmatrix}
\mathbf{\Lambda} \mathbf{\Lambda}_\varphi^{-2} \\
\mathbf{0}
\end{pmatrix} \mathbf{U}^\top \alpha(\mathcal{X},\mathcal{Y}).
\end{gathered}
\end{equation}
where $\alpha(\mathcal{X},\mathcal{Y}) = f_{\theta_0}(\mathcal{X}) - \mathcal{Y}$, and
\begin{equation}
\begin{aligned}
\mathbf{B} =& \left(\mathbb{E}_{x,y}\left[\alpha(x,y)^2\right] -\frac{1}{N}\alpha(\mathcal{X}, \mathcal{Y})^\top \alpha(\mathcal{X}, \mathcal{Y}) \right),\\
\mathbf{L} =& \left(\mathbb{E}_{x,y}\left[\alpha(x,y) \mathbf{J}(x) \right] - \frac{1}{N} \alpha(\mathcal{X}, \mathcal{Y})^\top \mathbf{J} \right) \mathbf{V},\\
\mathbf{Q} =& \mathbf{V}^\top \left(\mathbb{E}_{x}\left[\mathbf{J}(x)^\top \mathbf{J}(x)\right] - \frac{1}{N}\mathbf{J}^\top \mathbf{J}\right) \mathbf{V}.
\end{aligned}
\end{equation}
\end{thm}
\begin{proof}
Recall the definition of expected risk  and the expression of $f_\infty$ that
\begin{equation}
f_{\theta_t}(x) = f_{\theta_0}(x) -\mathbf{J}(x)  \mathbf{V} \begin{pmatrix}
\mathbf{\Lambda}^{-1} \left(\mathbf{I}_\varphi \oplus \mathbf{0}\right) \\
\mathbf{0}
\end{pmatrix} \mathbf{U}^\top\left(f_{\theta_0}(\mathcal{X}) - \mathcal{Y}\right).
\end{equation}
We have
\begin{equation}
\begin{aligned}
\mathcal{R}(f_{\theta_\infty})&= \mathbb{E}_{(x,y)\sim p_{data}(x,y)}\left[(f_{\theta_\infty}(x) - y)^2\right]\\
&= \mathbb{E}_{x}\left[\left(f_{\theta_0}(x) - y - J(x)  V \begin{pmatrix}
\Lambda^{-1} \left(\mathbf{I}_\varphi \oplus \mathbf{0}\right) \\
\mathbf{0}
\end{pmatrix} U^\top\left(f_{\theta_0}(\mathcal{X}) - \mathcal{Y}\right)\right)^2\right]\\
&= \mathbb{E}_{x}\left[\left(f_{\theta_0}(x) - y\right)^2\right]\\
&- 2\mathbb{E}_{x}\left[\left(f_{\theta_0}(x) - y\right) J(x)  V \begin{pmatrix}
\Lambda^{-1} \left(\mathbf{I}_\varphi \oplus \mathbf{0}\right) \\
\mathbf{0}
\end{pmatrix} U^\top\left(f_{\theta_0}(\mathcal{X}) - \mathcal{Y}\right)\right]\\
&+ \mathbb{E}_{x}\left[\left(J(x)  V \begin{pmatrix}
\Lambda^{-1} \left(\mathbf{I}_\varphi \oplus \mathbf{0}\right) \\
\mathbf{0}
\end{pmatrix} U^\top\left(f_{\theta_0}(\mathcal{X}) - \mathcal{Y}\right)\right)^2\right]\\
&\triangleq \mathbb{E}_{x}\left[\left(f_{\theta_0}(x) - y\right)^2\right] - T_1 + T_2.
\end{aligned}
\end{equation}
Since
\begin{equation}
\begin{aligned}
\mathbb{E}_{x,y}\left[\left(f_{\theta_0}(x) - y\right)^2\right]&= \mathbb{E}_{x,y}\left[\alpha(x,y)^2\right]\\
&= \frac{1}{N}\alpha(\mathcal{X}, \mathcal{Y})^\top \alpha(\mathcal{X}, \mathcal{Y}) + \left(\mathbb{E}_{x,y}\left[\alpha(x,y)^2\right] -\frac{1}{N}\alpha(\mathcal{X}, \mathcal{Y})^\top \alpha(\mathcal{X}, \mathcal{Y}) \right)\\
&\triangleq \frac{1}{N}\alpha(\mathcal{X}, \mathcal{Y})^\top \alpha(\mathcal{X}, \mathcal{Y}) + \mathbf{B}.
\end{aligned}
\end{equation}
where $\mathbf{B} = \left(\mathbb{E}_{x,y}\left[\alpha(x,y)^2\right] -\frac{1}{N}\alpha(\mathcal{X}, \mathcal{Y})^\top \alpha(\mathcal{X}, \mathcal{Y}) \right)$.
\par
For the second term at the right side $T_1$, we have
\begin{equation}
\begin{aligned}
T_1 &= 2\mathbb{E}_{x,y}\left[\left(f_{\theta_0}(x) - y\right) J(x)  V \begin{pmatrix}
\Lambda^{-1} \left(\mathbf{I}_\varphi \oplus \mathbf{0}\right) \\
\mathbf{0}
\end{pmatrix} U^\top\left(f_{\theta_0}(\mathcal{X}) - \mathcal{Y}\right)\right]\\
&= 2\mathbb{E}_{x,y}\left[\left(f_{\theta_0}(x) - y\right) J(x) \right] V \begin{pmatrix}
\Lambda^{-1} \left(\mathbf{I}_\varphi \oplus \mathbf{0}\right) \\
\mathbf{0}
\end{pmatrix} U^\top\left(f_{\theta_0}(\mathcal{X}) - \mathcal{Y}\right)\\
&= \frac{2}{N} \alpha(\mathcal{X}, \mathcal{Y})^\top J V \begin{pmatrix}
\Lambda^{-1} \left(\mathbf{I}_\varphi \oplus \mathbf{0}\right) \\
\mathbf{0}
\end{pmatrix} U^\top\alpha(\mathcal{X}, \mathcal{Y})\\
&+ 2\left(\mathbb{E}_{x,y}\left[\left(f_{\theta_0}(x) - y\right) J(x) \right] - \frac{1}{N} \alpha(\mathcal{X}, \mathcal{Y})^\top J \right) V \begin{pmatrix}
\Lambda^{-1} \left(\mathbf{I}_\varphi \oplus \mathbf{0}\right) \\
\mathbf{0}
\end{pmatrix} U^\top\alpha(\mathcal{X}, \mathcal{Y})\\
&= \frac{2}{N} \alpha(\mathcal{X}, \mathcal{Y})^\top U\left(\mathbf{I}_\varphi \oplus \mathbf{0}\right) U^\top\alpha(\mathcal{X}, \mathcal{Y})\\
&+ 2\left(\mathbb{E}_{x,y}\left[\left(f_{\theta_0}(x) - y\right) J(x) \right] - \frac{1}{N} \alpha(\mathcal{X}, \mathcal{Y})^\top J \right) V \begin{pmatrix}
\Lambda^{-1} \left(\mathbf{I}_\varphi \oplus \mathbf{0}\right) \\
\mathbf{0}
\end{pmatrix} U^\top\alpha(\mathcal{X}, \mathcal{Y})\\
&\triangleq \frac{2}{N} \alpha(\mathcal{X}, \mathcal{Y})^\top U\left(\mathbf{I}_\varphi \oplus \mathbf{0}\right) U^\top\alpha(\mathcal{X}, \mathcal{Y})
+ 2\mathbf{L}\begin{pmatrix}
\Lambda^{-1} \left(\mathbf{I}_\varphi \oplus \mathbf{0}\right) \\
\mathbf{0}
\end{pmatrix} U^\top\alpha(\mathcal{X}, \mathcal{Y}),
\end{aligned}
\end{equation}
where $\mathbf{L} = \left(\mathbb{E}_{x,y}\left[\left(f_{\theta_0}(x) - y\right) J(x) \right] - \frac{1}{N} \alpha(\mathcal{X}, \mathcal{Y})^\top J \right) V$. 
\par
The third term $T_2$ can be similarly rewirte as 
\begin{equation}
\begin{aligned}
T_2 &= \mathbb{E}_{x}\left[\left(J(x)  V \begin{pmatrix}
\Lambda^{-1} \left(\mathbf{I}_\varphi \oplus \mathbf{0}\right) \\
\mathbf{0}
\end{pmatrix} U^\top\left(f_{\theta_0}(\mathcal{X}) - \mathcal{Y}\right)\right)^2\right]\\
&= \alpha(\mathcal{X}, \mathcal{Y})^\top U \begin{pmatrix}
\Lambda^{-1}\left(\mathbf{I}_\varphi \oplus \mathbf{0}\right) & \mathbf{0}\end{pmatrix} V^\top
\mathbb{E}_{x}\left[J(x)^\top J(x)\right] 
 V \begin{pmatrix}
\Lambda^{-1} \left(\mathbf{I}_\varphi \oplus \mathbf{0}\right) \\
\mathbf{0}
\end{pmatrix} U^\top \alpha(\mathcal{X}, \mathcal{Y})\\
&= \frac{1}{N}\alpha(\mathcal{X}, \mathcal{Y})^\top U \begin{pmatrix}
\Lambda^{-1}\left(\mathbf{I}_\varphi \oplus \mathbf{0}\right) & \mathbf{0}\end{pmatrix}V^\top
J^\top J V \begin{pmatrix}
\Lambda^{-1} \left(\mathbf{I}_\varphi \oplus \mathbf{0}\right) \\
\mathbf{0}
\end{pmatrix} U^\top \alpha(\mathcal{X}, \mathcal{Y})\\
&+ \alpha(\mathcal{X}, \mathcal{Y})^\top U \begin{pmatrix}
\Lambda^{-1}\left(\mathbf{I}_\varphi \oplus \mathbf{0}\right) & \mathbf{0}\end{pmatrix}V^\top
\left(\mathbb{E}_{x}\left[J(x)^\top J(x)\right] - \frac{1}{N}J^\top J\right)
 V \begin{pmatrix}
\Lambda^{-1} \left(\mathbf{I}_\varphi \oplus \mathbf{0}\right) \\
\mathbf{0}
\end{pmatrix} U^\top \alpha(\mathcal{X}, \mathcal{Y})\\
&= \frac{1}{N}\alpha(\mathcal{X}, \mathcal{Y})^\top U \left(\mathbf{I}_\varphi \oplus \mathbf{0}\right) U^\top \alpha(\mathcal{X}, \mathcal{Y})\\
&+ \alpha(\mathcal{X}, \mathcal{Y})^\top U \begin{pmatrix}
\Lambda^{-1}\left(\mathbf{I}_\varphi \oplus \mathbf{0}\right) & \mathbf{0}\end{pmatrix}V^\top
\left(\mathbb{E}_{x}\left[J(x)^\top J(x)\right] - \frac{1}{N}J^\top J\right)
 V \begin{pmatrix}
\Lambda^{-1} \left(\mathbf{I}_\varphi \oplus \mathbf{0}\right) \\
\mathbf{0}
\end{pmatrix} U^\top\alpha(\mathcal{X}, \mathcal{Y})\\
&\triangleq \frac{1}{N}\alpha(\mathcal{X}, \mathcal{Y})^\top U \left(\mathbf{I}_\varphi \oplus \mathbf{0}\right) U^\top \alpha(\mathcal{X}, \mathcal{Y})\\
&+ \alpha(\mathcal{X}, \mathcal{Y})^\top U \begin{pmatrix}
\Lambda^{-1}\left(\mathbf{I}_\varphi \oplus \mathbf{0}\right) & \mathbf{0}\end{pmatrix}
\mathbf{Q}
\begin{pmatrix}
\Lambda^{-1} \left(\mathbf{I}_\varphi \oplus \mathbf{0}\right) \\
\mathbf{0}
\end{pmatrix} U^\top\alpha(\mathcal{X}, \mathcal{Y}),
\end{aligned}
\end{equation}
where $\mathbf{Q} = V^\top \left(\mathbb{E}_{x}\left[J(x)^\top J(x)\right] - \frac{1}{N}J^\top J\right) V$.
\par
Therefore, the generaliztion error can be split to two parts:
\begin{equation}
\mathcal{R}(f_{\theta_\infty})= \mathcal{R}_1 + \mathcal{R}_2
\end{equation}
with
\begin{equation}
\begin{aligned}
\mathcal{R}_1 &= \frac{1}{N}\alpha(\mathcal{X}, \mathcal{Y})^\top \alpha(\mathcal{X}, \mathcal{Y}) - \frac{1}{N} \alpha(\mathcal{X}, \mathcal{Y})^\top U\left(\mathbf{I}_\varphi \oplus \mathbf{0}\right) U^\top\alpha(\mathcal{X}, \mathcal{Y})\\
&= \frac{1}{N} \alpha(\mathcal{X}, \mathcal{Y})^\top U \left(\mathbf{I}_N - \left(\mathbf{I}_\varphi \oplus \mathbf{0}\right)\right) U^\top\alpha(\mathcal{X}, \mathcal{Y})\\
&= \frac{1}{N} \alpha(\mathcal{X}, \mathcal{Y})^\top U \left( \mathbf{I}_N - \Lambda \mathbf{\Lambda}_\varphi^{-2}\right) U^\top\alpha(\mathcal{X}, \mathcal{Y}).
\end{aligned}
\end{equation}
and
\begin{equation}
\begin{aligned}
\mathcal{R}_2 &= \mathbf{B} - 2\mathbf{L}\begin{pmatrix}
\Lambda^{-1} \left(\mathbf{I}_\varphi \oplus \mathbf{0}\right) \\
\mathbf{0}
\end{pmatrix} U^\top\alpha(\mathcal{X}, \mathcal{Y})\\
&+ \alpha(\mathcal{X}, \mathcal{Y})^\top U \begin{pmatrix}
\Lambda^{-1}\left(\mathbf{I}_\varphi \oplus \mathbf{0}\right) & \mathbf{0}\end{pmatrix}
\mathbf{Q}
\begin{pmatrix}
\Lambda^{-1} \left(\mathbf{I}_\varphi \oplus \mathbf{0}\right) \\
\mathbf{0}
\end{pmatrix} U^\top \alpha(\mathcal{X},\mathcal{Y})\\
&= \mathbf{B} - 2\mathbf{L}\begin{pmatrix}
\Lambda \mathbf{\Lambda}_\varphi^{-2} \\
\mathbf{0}
\end{pmatrix} U^\top\alpha(\mathcal{X}, \mathcal{Y})+ \alpha(\mathcal{X}, \mathcal{Y})^\top U \begin{pmatrix}
\Lambda \mathbf{\Lambda}_\varphi^{-2} & \mathbf{0}\end{pmatrix}
\mathbf{Q}
\begin{pmatrix}
\Lambda \mathbf{\Lambda}_\varphi^{-2} \\
\mathbf{0}
\end{pmatrix} U^\top \alpha(\mathcal{X},\mathcal{Y}).
\end{aligned}
\end{equation}
\end{proof}

\section*{Proof of Corollary 2}

\begin{coro}
For the result in Theorem \ref{GeneralizationThm}, if the training set were drawn i.i.d. from the true data distribution and the training set is large enough, the risk above can be written in orthogonal form:
\begin{equation}
\begin{aligned}
\mathcal{R}_1 =& \frac{1}{N} \sum\limits_{i=1}^N \left(1 - \lambda_i^2 \varphi\left(\frac{1}{\lambda_i^2}\right)\right) \left(\alpha(\mathcal{X},\mathcal{Y})^\top u_i\right)^2,\\
\mathcal{R}_2 \approx& \sum\limits_{i=1}^N \left(q_i \lambda_i^2 \varphi\left(\frac{1}{\lambda_i^2}\right)^2 \left(\alpha(\mathcal{X},\mathcal{Y})^\top u_i\right)^2 -2 l_i \lambda_i \varphi\left(\frac{1}{\lambda_i^2}\right) \alpha(\mathcal{X},\mathcal{Y})^\top u_i + b_i \right).
\end{aligned}
\end{equation}
and
\begin{equation}
\begin{aligned}
\mathcal{R}(f_\infty) \approx& \mathcal{R}_1 + \mathcal{R}_2.
\end{aligned}
\end{equation}
where $q_i$ is the diagonal elememts of $\mathbf{Q}$, $l_i, b_i$ is the $i_{th}$ element of $\mathbf{L}$ and $\mathbf{B}$, respectively, and $u_i$ is the $i_{th}$ colume of $\mathbf{U}$, i.e. the eigenvector of the NTK, Thus with the criteirion $\varphi(\lambda_i)$ defined as
\begin{equation}\label{criterion}
q_i \frac{1}{\lambda_i^2} - \frac{2l_i}{\alpha(\mathcal{X},\mathcal{Y})^\top u_i}\frac{1}{\lambda_i} -\frac{1}{N} > 0,
\end{equation}
the Modified NGD can reduce the generalization error.
\end{coro}
\begin{proof}
By the law of large number, we have
\begin{equation}
\mathbb{E}_{x}\left[\mathbf{J}(x)^\top \mathbf{J}(x)\right] - \frac{1}{N}\mathbf{J}^\top \mathbf{J} \to \mathbf{0}, \quad \text{as}\; N \to \infty,
\end{equation}
therefore,
\begin{equation}
\mathbf{V}^\top \mathbb{E}_{x}\left[\mathbf{J}(x)^\top \mathbf{J}(x)\right] \mathbf{V} \to  \frac{1}{N}\mathbf{V}^\top\mathbf{J}^\top \mathbf{J} \mathbf{V} = \frac{1}{N} \mathbf{\Lambda}^2.
\end{equation}
Notice that the RHS is a diagonal matrix, we approximate the LHS with a diagonal matrix when the amount of samples is large enough. Thus the matrix $\mathbf{Q}$ can be approximated by a diagonal matrix, and we denote the diagonal elements of $\mathbf{Q}$ as $q_i$ for $i = 1, \dots, N$. That is
\begin{equation}
\mathbf{Q} \approx \hat{\mathbf{Q}} = diag(q_1, \dots, q_N).
\end{equation}
And with denoting
\begin{equation}
\begin{aligned}
\mathbf{L} =& (l_1, \dots, l_N),\\
\mathbf{B} =& (b_1, \dots, b_N),
\end{aligned}
\end{equation}
we can rewrite $\mathcal{R}_2$ as 
\begin{equation}
\begin{aligned}
\mathcal{R}_2 =& \mathbf{B} - 2\mathbf{L}\begin{pmatrix}
\mathbf{\Lambda} \mathbf{\Lambda}_\varphi^{-2} \\
\mathbf{0}
\end{pmatrix} \mathbf{U}^\top\alpha(\mathcal{X}, \mathcal{Y})+ \alpha(\mathcal{X}, \mathcal{Y})^\top \mathbf{U} \begin{pmatrix}
\mathbf{\Lambda} \mathbf{\Lambda}_\varphi^{-2} & \mathbf{0}\end{pmatrix}
\mathbf{Q}
\begin{pmatrix}
\mathbf{\Lambda} \mathbf{\Lambda}_\varphi^{-2} \\
\mathbf{0}
\end{pmatrix} \mathbf{U}^\top \alpha(\mathcal{X},\mathcal{Y})\\
\approx& \mathbf{B} - 2\mathbf{L}\begin{pmatrix}
\mathbf{\Lambda} \mathbf{\Lambda}_\varphi^{-2} \\
\mathbf{0}
\end{pmatrix} \mathbf{U}^\top\alpha(\mathcal{X}, \mathcal{Y})+ \alpha(\mathcal{X}, \mathcal{Y})^\top \mathbf{U} \begin{pmatrix}
\mathbf{\Lambda} \mathbf{\Lambda}_\varphi^{-2} & \mathbf{0}\end{pmatrix}
\hat{\mathbf{Q}}
\begin{pmatrix}
\mathbf{\Lambda} \mathbf{\Lambda}_\varphi^{-2} \\
\mathbf{0}
\end{pmatrix} \mathbf{U}^\top \alpha(\mathcal{X},\mathcal{Y})\\
=& \sum\limits_{i=1}^N \left(q_i \lambda_i^2 \varphi\left(\frac{1}{\lambda_i^2}\right)^2 \left(\alpha(\mathcal{X},\mathcal{Y})^\top u_i\right)^2 -2 l_i \lambda_i \varphi\left(\frac{1}{\lambda_i^2}\right) \alpha(\mathcal{X},\mathcal{Y})^\top u_i + b_i \right).
\end{aligned}
\end{equation}
And $\mathcal{R}_1$ can be rewrite as 
\begin{equation}
\begin{aligned}
\mathcal{R}_1 =& \frac{1}{N} \alpha(\mathcal{X}, \mathcal{Y})^\top \mathbf{U} \left( \mathbf{I}_N - \mathbf{\Lambda}^2 \mathbf{\Lambda}_\varphi^{-2}\right) \mathbf{U}^\top\alpha(\mathcal{X}, \mathcal{Y}),\\
=&\frac{1}{N} \sum\limits_{i=1}^N \left(1 - \lambda_i^2 \varphi\left(\frac{1}{\lambda_i^2}\right)\right) \left(\alpha(\mathcal{X},\mathcal{Y})^\top u_i\right)^2.
\end{aligned}
\end{equation}
Therefore
\begin{equation}
\begin{aligned}
\mathcal{R}(f_\infty) =& \mathcal{R}_1 + \mathcal{R}_2\\
\approx& \sum\limits_{i=1}^N \left(q_i \lambda_i^2 \varphi\left(\frac{1}{\lambda_i^2}\right)^2 \left(\alpha(\mathcal{X},\mathcal{Y})^\top u_i\right)^2 -2 l_i \lambda_i \varphi\left(\frac{1}{\lambda_i^2}\right) \alpha(\mathcal{X},\mathcal{Y})^\top u_i - \frac{1}{N}\lambda_i^2 \varphi\left(\frac{1}{\lambda_i^2}\right) \left(\alpha(\mathcal{X},\mathcal{Y})^\top u_i\right)^2 \right.\\
&+ \left.  b_i + \frac{1}{N}\left(\alpha(\mathcal{X},\mathcal{Y})^\top u_i\right)^2\right).
\end{aligned}
\end{equation}
For each eigen component, the criterion must satisfy the following condition:
\begin{equation}
\begin{aligned}
&q_i \frac{1}{\lambda_i^2} \left(\alpha(\mathcal{X},\mathcal{Y})^\top u_i\right)^2 -2 l_i \frac{1}{\lambda_i} \alpha(\mathcal{X},\mathcal{Y})^\top u_i -  \frac{1}{N}\left(\alpha(\mathcal{X},\mathcal{Y})^\top u_i\right)^2 \\
&+ b_i + \frac{1}{N}\left(\alpha(\mathcal{X},\mathcal{Y})^\top u_i\right)^2 >  b_i + \frac{1}{N}\left(\alpha(\mathcal{X},\mathcal{Y})^\top u_i\right)^2.
\end{aligned}
\end{equation}
The LHS of the above inequality is the generalization error of preserving the training direction in eigenspace $span(u_i)$, and the RHS is the generalization error of modifying and muting the training direction in $span(u_i)$. Rearrange the above inequality to obtain:
\begin{equation}
q_i \frac{1}{\lambda_i^2} - \frac{2l_i}{\alpha(\mathcal{X},\mathcal{Y})^\top u_i}\frac{1}{\lambda_i} -\frac{1}{N} > 0.
\end{equation}
\end{proof}

\section*{Numerical Experiments}
This section aims to illustrate our theoretical results of Modified NGD, that is, based on the theoretical criterion of modification, the Modified NGD can reduce the generalization error compared with ordinary NGD and NGD with modification on other directions.
\par
Due to the high dimension of Fisher, all of our experiments are implemented on a two layers MLP (Multi-Layer Perceptron) with synthetic data\footnote{All codes, data and results can be find at \url{https://github.com/21veu/modified_NGD}.}. However, with the dicussions of the discrepancy bounds of NTK regime and general neural network \cite{jacot2018neural, arora2019exact, rudner2019natural}, our theoretical and numerical results can be generalized to general DNN.
\par
\textbf{Setup} We firstly draw samples uniformly from interval $[0,1)$, then split the samples to training set with 256 samples, validation set with 64 samples and test set with 64 samples, and apply perturbation to the training set:
\begin{equation}
x \to xe^{-\frac{(1-x)^2}{\sigma^2}}
\end{equation}
with different perturbation factor $\sigma^2$.
\par
For a function approximation problem:
\begin{equation}
f^\star(x) = \cos x \sin x.
\end{equation}
We use a two layers MLP with $2^{12}$ neurons with He initialization \cite{he2015delving} to train on the training set. We perform two optimization algorithms: modified NGD and NGD with all the other settings being same. Modified NGD uses validation set for the true distribution computation in the criterion \ref{criterion} to decide the directions to be modified. The initial learning rate is set as 0.1 with learning rate half decay and train for 500 epochs. We run each experiments for 20 random seed. and the results are reported on the average of different random seeds. We implements the numerical experiments for different degrees of perturbation with the mean of perturbed data changing roughly equally, thus we choose the perturbation factors $\sigma^2$ to be: 10, 5, 1.5 and 1.

\textbf{Results} Firstly, we observed from the experimental results that modified NGD is more stable than original NGD. We list the mean test error of Modified NGD and original NGD in the last 10 epochs of different random seeds with different pertubation factor $\sigma^2$ in Table \ref{table1}. In Table \ref{table1}, the test loss exceeding \textbf{0.05} is highlighted in bold type. From the results in Table \ref{table1}, we can see that for all setting and all random seed, modified NGD can achieve the optimal solution, but the original NGD might be stacked by local minima. 
\begin{table}[h]
	\centering
	\caption{Mean test error of Modified NGD and original NGD in the last 10 epochs of different random seeds with different pertubation factor $\sigma^2$.}
	\label{table1}
	\scalebox{0.8}{
	\renewcommand\arraystretch{2}
	\begin{tabular}{crrrrrrrr}
		\hline
		\multirow{2}{*}{\diagbox[height = 0.11\linewidth]{Seed}{Test Error}{Method}}     & \multicolumn{4}{c}{Modified NGD} & \multicolumn{4}{c}{Original NGD}     \\ 
		\cline{3-4} \cline{7-8}
		 & $\sigma^2=1$        & $\sigma^2=1.5$        & $\sigma^2=5$ & $\sigma^2=10$  & $\sigma^2=1$ & $\sigma^2=1.5$ & $\sigma^2=5$ & $\sigma^2=10$  \\ 
		\hline
		1 & 0.0208  &  0.0279  &  0.0229  &  0.0190  &  0.0274  &  0.0335  &   0.0348 &  0.0255 \\ 
		
		2 & 0.0250  &  0.0277  &  0.0240  &  0.0193  & \textbf{0.4501}   &  \textbf{0.3444}  &   \textbf{0.3700} & \textbf{0.3732}  \\ 
		
		3 & 0.0220  &  0.0270  &  0.0224  &  0.0177  &  0.0224  &  0.0384  &  0.0337  &  0.0230 \\ 
		
		4 &0.0203   & 0.0271   &  0.0223  & 0.0178   &  0.0288  &  0.0333  &  0.0325  &  0.0254 \\ 
		
		5 & 0.0220  &   0.0270 &  0.0234  &  0.0192  & 0.0262   &   0.0285 &   0.0247 &  0.0202 \\ 
		
		6 &  0.0219 &   0.0274 &  0.0226  &  0.0182  &  \textbf{1.1343}  &  \textbf{1.3745}  &   \textbf{1.2680} & \textbf{1.2683}  \\ 
		
		7 &  0.0217 &  0.0268  &  0.0233  &  0.0179  &  \textbf{0.1092}  & \textbf{0.1149}   &   \textbf{0.1210} &  \textbf{0.0890} \\ 
		
		8 &  0.0214 &  0.0288  &  0.0224  &   0.0179 &  0.0459  &  \textbf{0.0776}  &  \textbf{0.0673}  &  \textbf{0.0517} \\ 
		
		9 &  0.0216 &  0.0267  &  0.0224  &  0.0179  &  \textbf{0.0817}  &  \textbf{0.0998}  &  0.0350  &  0.0305 \\ 
		
		10 & 0.0204  & 0.0270   &  0.0233  &  0.0188  &  \textbf{0.3018}  &  \textbf{0.3840}  &   \textbf{0.3316} &  \textbf{0.3506} \\ 
		
		11 &  0.0222 &  0.0272  &  0.0239  & 0.0181   &  \textbf{0.1003}  &   \textbf{0.1662} &   \textbf{0.1529} &  0.0211 \\ 
		
		12 & 0.0213  &  0.0270  & 0.0223   &  0.0179  &  \textbf{0.4897}  & \textbf{0.5625}   &\textbf{0.5224}    & \textbf{0.5421}  \\ 
		
		13 & 0.0210  & 0.0269   & 0.0226   & 0.0178   &  0.0211  &  0.0274  & 0.0225   &0.0186   \\ 
		
		14 & 0.0210  & 0.0267   &  0.0222  &  0.0178  &  0.0374  & \textbf{0.0589}   &  \textbf{0.0585}  & 0.0381  \\ 
		
		15 & 0.0210  & 0.0268   &  0.0222  & 0.0178   & 0.0220   & 0.0273   & 0.0227   & 0.0189  \\ 
		
		16 & 0.0209  & 0.0271   & 0.0235   &  0.0177  & \textbf{0.0778}   &  \textbf{0.1555}  & \textbf{0.0506}   & 0.0306  \\ 
		
		17 & 0.0212  & 0.0266   &  0.0222  & 0.0178   & 0.0213   & 0.0266   &  0.0222  & 0.0179  \\ 
		
		18 & 0.0214  &  0.0269  & 0.0226   & 0.0177   & 0.0419   & 0.0417   &  0.0362  & 0.0308  \\ 
		
		19 & 0.0230  & 0.0266   & 0.0233   &  0.0178  & \textbf{1.7725}   & \textbf{1.6524}   &   \textbf{1.6050} & \textbf{1.7425}  \\ 
		
		20 & 0.0204  &  0.0306  &  0.0247  &  0.0187  & \textbf{1.3554}   &  \textbf{1.6067}  & \textbf{1.4616}   & \textbf{1.4980}  \\ 
		\hline
	\end{tabular}}
\end{table}
\par
As shown in Fig.\ref{sigmas}, with different degrees of perturbation on training data, the Modified NGD is stable and apparently of smaller generalization error than NGD. In the plots, a line represents the mean on random seeds and the envelope around it reflects 0.3 times standard deviation. 
\par 
Fig \ref{diffloss} illustrates the trendency of the difference of NGD and Modified NGD at the convergence point. As $\sigma^2$ decreases, the degree of perturbation on training data increases, then the generalization error of NGD increase, the performance of NGD gets worse than Modified NGD. 
\par
To varify our theoretical results of generalization decomposition, we implement a comparative experiment to illustrate that the criterion \ref{criterion} derived from our decomposition is more effective than other criterions. Inspired by many existing results such as \cite{bartlett2002rademacher, keskar2016large}, a view believes that cutting small eigenvalues benefits generalization. We trained on the training set perturbed with perturbation factor $\sigma^2 = 1$ by Modified NGD and NGD cut as many as eigenvalues but small ones. The test loss of these two algorithms are shown in Fig. \ref{cutsmall}. The experimental results demonstrate that cutting only small eigenvalues is effective to reduce the generalization error, but not as good as our criterion \ref{criterion}, which verified our theoretical results and corrected the conventional view.

\begin{figure}[h]
	\begin{minipage}{0.24\linewidth}
		\centering
  		\includegraphics[width=\textwidth]{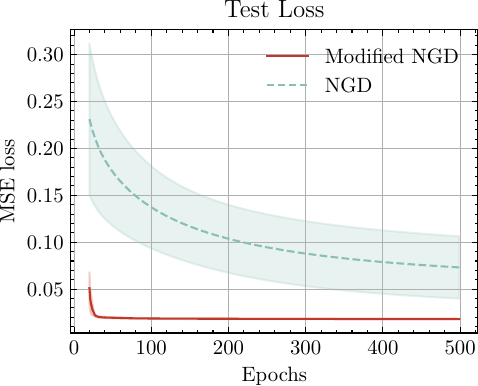}
  		\subcaption{$\sigma^2=10$}
  		\label{testloss10}
	\end{minipage}
	\begin{minipage}{0.24\linewidth}
  		\includegraphics[width=\textwidth]{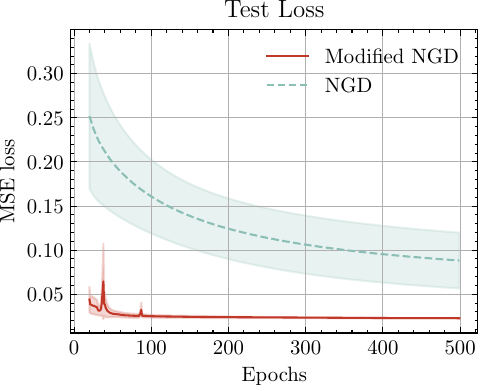}
  		\subcaption{$\sigma^2=5$}
  		\label{testloss5}
	\end{minipage}
	\begin{minipage}{0.24\linewidth}
  		\includegraphics[width=\textwidth]{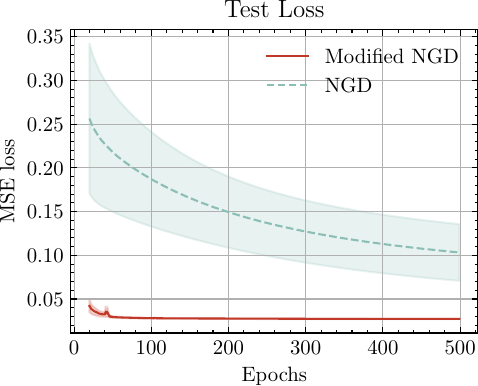}
  		\subcaption{$\sigma^2=1.5$}
  		\label{testloss1.5}
  		
	\end{minipage}
	\begin{minipage}{0.24\linewidth}
  		\includegraphics[width=\textwidth]{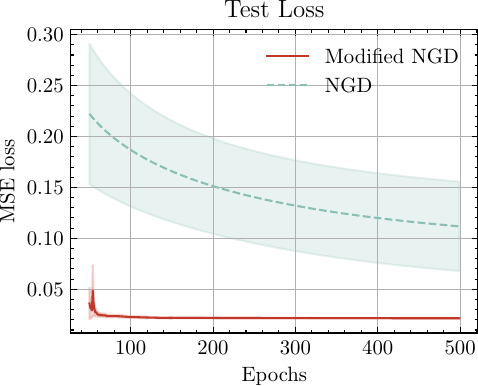}
  		\subcaption{$\sigma^2=1$}
  		\label{testloss1}
	\end{minipage}
	\caption{The test loss of NGD and Modified NGD with different degrees of perturbation during 500 epochs. \ref{testloss10} is the test loss results trained on training data perturbed with the perturbation factor $\sigma^2=10$; \ref{testloss5} with the perturbation factor $\sigma^2=5$; \ref{testloss1.5} with the perturbation factor $\sigma^2=1.5$; \ref{testloss1} with the perturbation factor $\sigma^2=1$. }
	\label{sigmas}
\end{figure}

\begin{figure}[h]
		\centering
  		\includegraphics[width=0.5\textwidth]{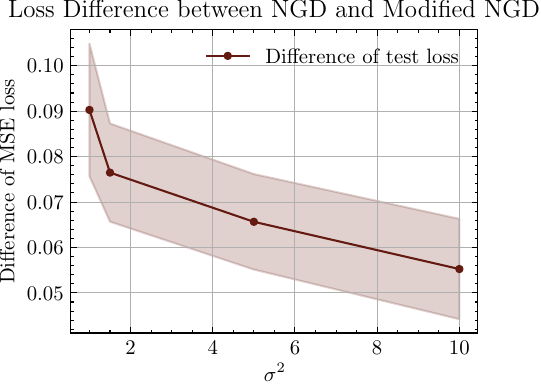}
  		\caption{The average difference of test loss of NGD and Modified NGD in the last 10 epochs with respect to different perturbation factor $\sigma^2$.}
  		\label{diffloss}
\end{figure}
\begin{figure}[h]
		\centering
		\includegraphics[width=0.45\textwidth]{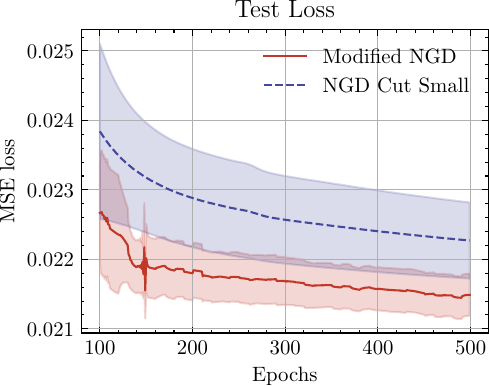}\quad
  		\includegraphics[width=0.5\textwidth]{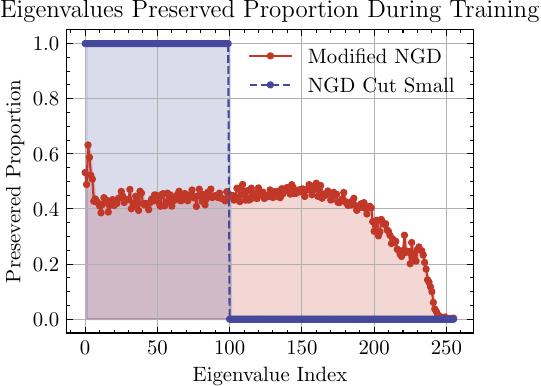}
  		\caption{The left subfigure shows the test loss of NGD that cut eigenvalues as many as Modified NGD but small ones and Modified NGD with perturbation factor $\sigma^2=1$ during 500 epochs; the right subfigure shows the proportion of each eigenvalue being preserved during 500 epochs}
		\label{cutsmall}
	
\end{figure}

As shown in Fig.\ref{diffloss} we can observe from the difference of the losses between NGD and modified NGD in the last epochs that with the $\sigma^2$ increasing, the difference of loss decreases. The increasing of $\sigma^2$ indicates the perturbation on training set is decreasing, thus the difference between training set and test set is getting smaller, the training direction of ordinary NGD is getting close to modified NGD.

\section*{Insights on existing generalization enhancing algorithms}\label{insights}
Since the machine learning achieved good performance on a lot of tasks, several algorithms aiming for enhancing the generalization performance based on existing resources were proposed. Our results can shed light on why these algorithms work theoretically. In the following, we give a few examples.
\par
\textbf{Cross domain generalization} Zeyi Huang et al. \cite{huang2020self} proposed a intuitive algorithm to enhance the performance of crosss domain generalization by cut the largest components of the gradient. The details of Zeyi Huang et al.'s algorithm can be found in \cite{huang2020self}. When the output layer is linear, we can simplify their algorithm in the NTK regime:
\begin{equation}
\begin{aligned}
\Delta \theta = -\eta \nabla_\theta \mathcal{L}(\hat{f}(\mathcal{X}), \mathcal{Y}) = -\eta \frac{\partial \hat{f}(\mathcal{X})}{\partial \theta} (f_{\theta_t}(\mathcal{X}) - \mathcal{Y}).
\end{aligned}
\end{equation}
Where $\hat{f}(\mathcal{X})$ is the modified feature of cutting the last layer feature with the largest original gradient. Since the output layer is linear, modification on the last feature layer  can be linearly transferred to modification in the output layer. And with SVD decomposition, the modification can be regarded as modifying the singular values of Jacobian matrix.
\begin{equation}
\mathbf{J}_\varphi = \frac{\partial \hat{f}(\mathcal{X})}{\partial \theta} = \mathbf{U} \begin{pmatrix}
\mathbf{\Lambda}_\varphi & \mathbf{0}
\end{pmatrix}  \mathbf{V}^\top.
\end{equation}
Here, the criterion is to mute the largest gradient.
\par
With this expression of modified Jacobian, we can similarly derive the solution of modified Gradient descent:
\begin{equation}
\frac{\partial f_{\theta_t}(\mathcal{X})}{\partial t} = - \eta \mathbf{J} \mathbf{J}_\varphi^\top (f_{\theta_t}(\mathcal{X}) - \mathcal{Y}).
\end{equation}
The solution on training set is 
\begin{equation}
f_{\theta_t}(\mathcal{X}) = \mathcal{Y} +  \mathbf{U} e^{-\eta \mathbf{\Lambda}^2 t}\left(\mathbf{I}_\varphi \oplus \mathbf{0}\right) \mathbf{U}^\top \left(f_{\theta_0}(\mathcal{X}) - \mathcal{Y}\right).
\end{equation}
And the solution on test point as $t \to \infty$ is:
\begin{equation}
f_{\theta_t}(x) = f_{\theta_0}(x) -\mathbf{J}(x)  \mathbf{V} \begin{pmatrix}
\mathbf{\Lambda}^{-1} \left(\mathbf{I}_\varphi \oplus \mathbf{0}\right) \\
\mathbf{0}
\end{pmatrix} \mathbf{U}^\top\left(f_{\theta_0}(\mathcal{X}) - \mathcal{Y}\right).
\end{equation}
Notice that the above solution is as the same as the solution of modified NGD of our framework, which means modifying Jacobian is implicitly implementing modification on Fisher, and consequently modifying the training directions of neural network in function space. . Therefore, our theoretical analyses of generalization decomposition and generalization reduction with criterion can be applied to their experimental algorithm.
And we can demonstrate that this modification is equivalent to cut the largest eigenvalues of empirical Fisher informantion matrix, 
\par
\textbf{Self distillation} Self distillation is a post-training method. Mobahi et al. \cite{mobahi2020self} shows that self distillation amplifies regularization effect at each distillation round, which make the eigenvalues of the Gram matrix of the kernel of the regularizer evolve. And after several distillation rounds, the new corresponding kernel's Gram matrix possesses smaller eigenvalues, thus enhances the generalization performance. They shown in \cite{mobahi2020self} that the solution of the regularized optimization problem after $t$ rounds distillation is
\begin{equation}
f_t^\star(x) = \mathbf{g}_x^\top \mathbf{U}^\top \mathbf{\Lambda}^{-1}\prod\limits_{i=0}^t \left(\mathbf{\Lambda}(c_i \mathbf{I} + \mathbf{\Lambda})^{-1}\right)\mathbf{U} \mathcal{Y}.
\end{equation}
where $\mathbf{g}$ is the Green function of the regularizer ,$c_i$ are the regulaization parameters, and $\mathbf{U}, \mathbf{\Lambda}$ are the eigendecomposition of the Gram matrix of the Green function that:
\begin{equation}
\mathbf{G} = \mathbf{U} \mathbf{\Lambda} \mathbf{U}^\top.
\end{equation}
And Mobahi et al. \cite{mobahi2020self} proved that the solution after $t$ rounds distillation can be regarded equally as the solution of a modified kernel without distillation:
\begin{equation}\label{distillationSolution}
f_t^\star(x) = {\mathbf{g}^\top_x}^\dagger \left( c_0 \mathbf{I} + \mathbf{G}^\dagger\right)^{-1} \mathcal{Y}. 
\end{equation}
with the eigenvalues of $\mathbf{G}^\dagger$ being
\begin{equation}
\lambda_k^\dagger = c_0 \frac{1}{\frac{\prod_{i=0}^t (\lambda_k+c_i)}{\lambda_k^{t+1}} -1}
\end{equation}
Compared \cite{mobahi2020self} with the solution of modified NGD, we can observe that the modified Gram matrix in \cite{mobahi2020self} has the similar role of the modified Fisher matrix in modified NGD. And Mobahi et al. \cite{mobahi2020self} proved that the eigenvalues $\lambda_k^\dagger$ of modified Gram matrix is descending as $t$ increasing. Therefore, in our framework, the eigenvalues of Gram matrix indicate the training dynamics in eigenspace, and consequently, self distillation employs a mild modification on training directions in function space introduced by the kernel. 
\par
\textbf{Small batch training and flat minima} Keskar et al. \cite{keskar2016large} proposed a metric to measure the sharpness of local minima:
\begin{equation}
\phi_{x,f}(\epsilon, A) = \frac{\max_{y\in\mathcal{C}_\epsilon} f(x+Ay) - f(x)}{1+f(x)}\times 100
\end{equation}
where $\mathcal{C}_\epsilon$ is some constraint set related to $A$ and $\epsilon$. From the defineition of this sharpness metric, it can be observed that when $A = \mathbf{I}$, the metric relates to the largest eigenvalue of Hessian $\nabla^2 f$, and when $A$ is randomly sampled it approximates the Ritz value of Hessian projected onto the column-space of $A$ \cite{keskar2016large}. As discussed in \cite{martens2020new} and \cite{liu2020linearity}, in overparameterized neural network, the Hessian is an approximation for Fisher information matrix, which indicates that modifying the training direction by modifying the eigenvalues of Fisher information matrix will change the sharpness of the convergence point in the function space. Thus with our results, it can be proved that flat minima corresponds to convergence point with lower generalization error.


\end{document}